\documentclass{article}
\usepackage{ijcai23}

\usepackage{times}
\usepackage{soul}
\usepackage{url}
\usepackage[hidelinks]{hyperref}
\usepackage[utf8]{inputenc}
\usepackage[small]{caption}
\usepackage{graphicx}
\usepackage{amsmath}
\usepackage{amsthm}
\usepackage{booktabs}
\usepackage{algorithm}
\usepackage{algorithmic}
\usepackage[switch]{lineno}
\usepackage{amsfonts}
\usepackage{bm}
\usepackage{multibib}
\usepackage{mathrsfs}
\newcites{app}{Appendix References}
\newtheorem{innercustomgeneric}{\customgenericname}
\providecommand{\customgenericname}{}
\newcommand{\newcustomtheorem}[2]{%
  \newenvironment{#1}[1]
  {%
   \renewcommand\customgenericname{#2}%
   \renewcommand\theinnercustomgeneric{##1}%
   \innercustomgeneric
  }
  {\endinnercustomgeneric}
}
\newcustomtheorem{customtheorem}{Theorem}
\newcustomtheorem{customcorollary}{Corollary}
\newcustomtheorem{customproposition}{Proposition}
\newcustomtheorem{customlemma}{Lemma}
\usepackage{thmtools,thm-restate}
\usepackage{multirow}
\usepackage{booktabs}
\usepackage{color,colortbl}

\definecolor{Gray}{gray}{0.9}

\newtheorem{theorem}{Theorem}

\newtheorem{lemma}{Lemma}
\theoremstyle{definition}

\title{Latent Processes Identification From Multi-View Time Series}

\author{
Zenan Huang$^{1,4}$
\and
Haobo Wang$^4$\and
Junbo Zhao$^4$\And
Nenggan Zheng\footnote{Corresponding author.}$^{1,2,3,4}$
\affiliations
$^1$Qiushi Academy for Advanced Studies (QAAS), Zhejiang University\\
$^2$the State Key Lab of Brain-Machine Intelligence, Zhejiang University\\
$^3$CCAI by MOE and Zhejiang Provincial Government (ZJU)\\
$^4$College of Computer Science and Technology, Zhejiang University\\
\emails
\{lccurious, wanghaobo, j.zhao, zng\}@zju.edu.cn,
}

\begin{document}

\maketitle

\begin{abstract}
Understanding the dynamics of time series data typically requires identifying the unique latent factors for data generation, \textit{a.k.a.}, latent processes identification. Driven by the independent assumption, existing works have made great progress in handling single-view data. However, it is a non-trivial problem that extends them to multi-view time series data because of two main challenges: (i) the complex data structure, such as temporal dependency, can result in violation of the independent assumption; (ii) the factors from different views are generally overlapped and are hard to be aggregated to a complete set. In this work, we propose a novel framework MuLTI that employs the contrastive learning technique to invert the data generative process for enhanced identifiability. Additionally, MuLTI integrates a permutation mechanism that merges corresponding overlapped variables by the establishment of an optimal transport formula. Extensive experimental results on synthetic and real-world datasets demonstrate the superiority of our method in recovering identifiable latent variables on multi-view time series.
\end{abstract}

\section{Introduction}
Detecting causal relationships from time series based on observations is a challenging problem in many fields of science and engineering \cite{spirtesCausationPredictionSearch2000}.
A thorough grasp of causal relationships, interaction pathways, and time lags is valuable for interpreting and modeling temporal processes \cite{pearlCausalityModelsReasoning2000}.
Existing works \cite{chickeringOptimalStructureIdentification2002,tsamardinosMaxminHillclimbingBayesian2006,zhangCompletenessOrientationRules2008,hoyerNonlinearCausalDiscovery2009,zhengDAGsNOTEARS2018} typically rely on predefined variables. 
However, such a strategy is not directly applicable to real-world scenarios, where data are intertwined with unknown generation processes, and causal variables are not readily available.
Therefore, identifying the latent sources is crucial for interpreting underlying causal relations and elucidating the genuine dynamics inherent in the temporal data.

\begin{figure}[t]
    \centering
    \includegraphics[width=0.8\linewidth]{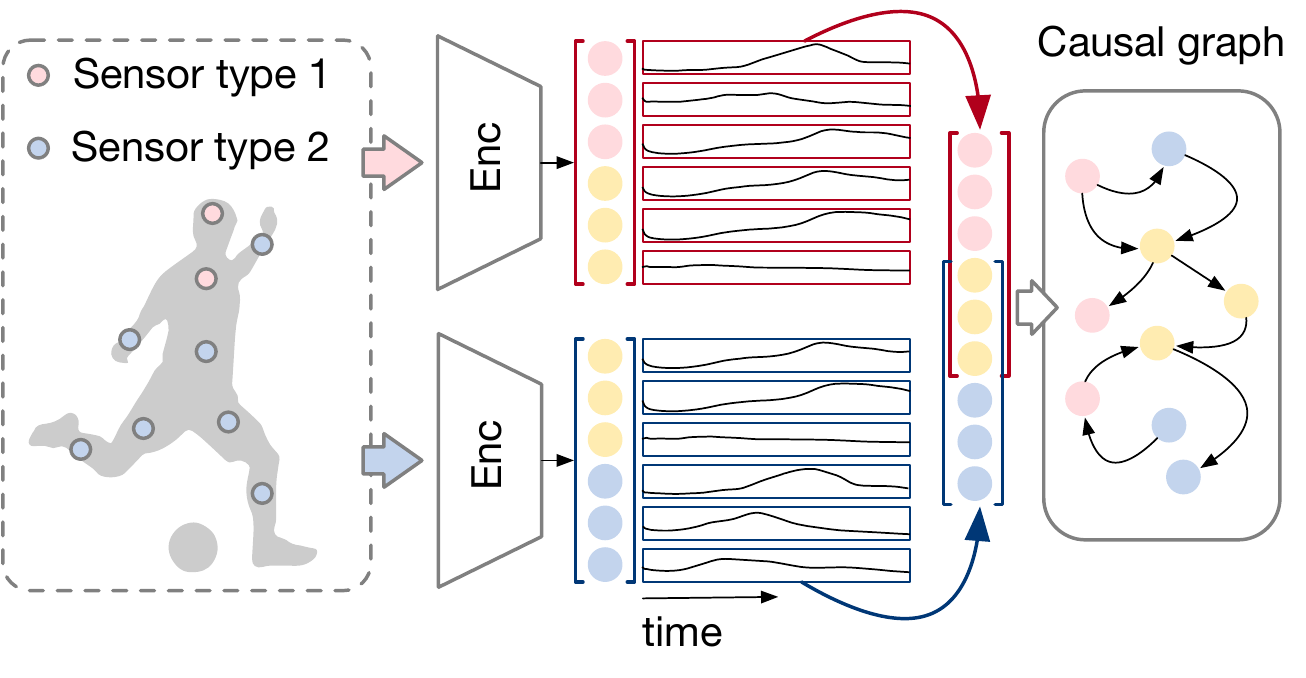}
    \caption{The diagram depicts a pipeline analyzing multi-view physiological time-series data. This pipeline learns temporal embeddings from both views, aligns variables considering their dependencies, and effectively reveals underlying variables and relationships.}
    \label{fig:motivation}
\end{figure}

To cope with this problem, non-linear independent component analysis (nICA) \cite{hyvarinenNonlinearIndependentComponent1999} has shown promising results in terms of identifiability, by effectively exploiting the underlying structure of the data. 
On time series data, most nICA techniques attempt to utilize the temporal structure among sequential observations, such as temporal contrastive learning \cite{DBLP:conf/nips/HyvarinenM16}, permutation contrastive learning \cite{hyvarinenNonlinearICATemporally2017}, and generalized contrastive learning \cite{hyvarinenNonlinearICAUsing2019}. 
However, these methods inherently assume the latent components are independent, which may not hold in practice.
Moreover, in broader real-world scenarios, causal discovery tasks mostly require identifying dependent relations from extensive data collections, which may further restrict the identifiability of nICA techniques.
Take Figure~\ref{fig:motivation} as an example, understanding the physiological processes of human actions requires jointly analyzing the data across multiple kinds of sensors and a physiological signal may influence multiple downstream regions with different time lags.
Similar cases also exist in a wide variety of real-world scenarios, such as multi-view sensors information fusion \cite{heCSMVCMultiviewMethod2021}, multi-market stock index analysis \cite{wangIncorporatingExpertBasedInvestment2020}, \textit{etc}.
Hitherto, few efforts have been made to address such challenging yet realistic problems.

The example above, which we call the multi-view latent process identification (MVLPI) problem, poses two combinatorial challenges.
First, the diverse time delays in interactions among variables in time series data complicate direct estimation of latent variables.
Second, different views typically correlate with subsets of latent factors.
Due to identifiability issues, merely inverting these views does not guarantee alignment of the recovered factors.
Hence, aggregating them to the final complete latent variables is non-trivial.

In this paper, we propose the \textbf{Mu}lti-view \textbf{L}aten\textbf{T} Processes \textbf{I}dentification (dubbed \textbf{MuLTI}) framework that learns identifiable causal related variables from the multi-view data.
MuLTI identifies latent variables from multi-view observations using three strategies:
1) It reformulates causally related process distributions using conditional independence, replacing original latent variables with independent causal process noises;
2) It applies contrastive learning to maximize mutual information between prior and posterior conditional distributions;
3) It aligns partially overlapping latent variables using learnable permutation matrices, optimized with Sinkhorn method.
The first two strategies ensure the learning of identifiable latent variables, while the last one bridges the causal relations across multiple views.
Given that only a subset of each view's source components correspond, matching and merging shared latent variables resemble the sorting and alignment of top-k shared components among estimated variables.

We evaluate our method on both synthetic and real-world data, spanning multivariable time series and visual tasks.
Experiment results show that the latent variables are reliably identified from observational multi-view data.
To the best of our knowledge, the learning of causal dependent latent variables from multi-view data has no prior solution.
The proposed framework may also serve as a factor analyzing tool for extracting reliable features for downstream tasks of multi-view learning.
More theoretical and empirical results can be found in Appendix.

\section{Methodology}
\subsection{Problem Formulation}

In contrast to the single-view setup, the MVLPI problem presents a unique challenge: only a portion of the latent factors can be recovered from an individual view.
Thus, it is necessary to revisit the data-generating procedure of the MVLPI task.
As shown in Figure~\ref{fig:motivation-assumption}, we consider the MVLPI task on the time series data, with the goal of recovering the latent factor $\bm{z}_{t}\in \mathbb{R}^{d}$ at each time step $t$, which uniquely generates the observed views. To achieve this, we make a common assumption that the current state is spatiotemporally dependent on the historical states,
\begin{equation}
\bm{z}_{i,t}=f_{i}(\textbf{Pa}(z_{i,t}), \epsilon_{i,t}),
\label{eq:spatial-temporal}
\end{equation}
where $\textbf{Pa}(z_{i,t})$ denotes causal parents of $i$-th factor at current step $t$, $\epsilon_{i,t}$ denotes a noise component of causal transition.
While $\bm{z}_t$ contains a complete set of latent factors that we are truly interested in, we note that each observed view may be dependent on only a subset of it. 
Formally, we have the following data-generating process,
\begin{equation}
    \bm{x}^{1}_{t}  = g^{1}(\bm{z}^{1}_{t}), \quad 
    \bm{x}^{2}_{t}  = g^{2}(\bm{z}^{2}_{t}), \quad 
    \text{where}~~\bm{z}_{t} = \bm{z}^{1}_{t}\cup \bm{z}^{2}_{t},
\end{equation}
where $\bm{x}^{v}_{t}$ is the $v$-th observed view at time step $t$, and $g^{v}(\cdot)$ is the corresponding map function.
Note that here we consider a 2-view setup for the sake of brevity, while our framework can be easily generalized to handle more views. 
\begin{figure}[t]
    \centering
    \includegraphics[width=0.9\linewidth]{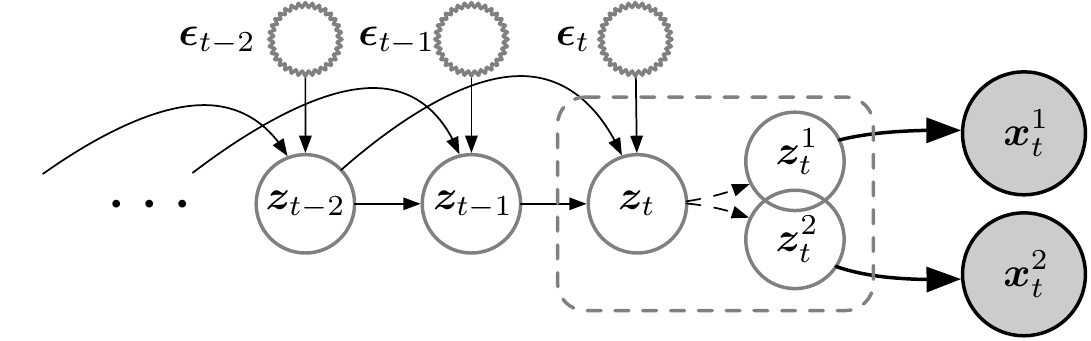}
    \caption{Graphical model of data generation. Each view may be generated by part of latent variables individually.}
    \label{fig:motivation-assumption}
\end{figure}

Following the above definition, we can identify the factors from the observed views, and aggregate them into the complete version $\bm{z}_{t}$.
However, as each view performs a distinctive nonlinear mixing of the source factors, it is challenging to identify the corresponding latent variables entangled in views, or to ensure the factors from different views are in the same subspace.
Second, though we slightly abuse the notation `$\cup$' to show the \textit{aggregation} relation of view-specific and the complete factors, it is non-trivial to reconstruct the complete $\bm{z}_{t}$ from a set of indeterminate estimations in practice.
To remedy this problem, we introduce our novel \textbf{MuLTI} framework which comprises a contrastive learning module for enhanced identifiability alongside a novel merge operator that aggregates the view-specific latent variables to a complete one.
The whole procedure is illustrated in Figure~\ref{fig:framework}.

\subsection{Contrastive Learning Module}
\label{sec:contrastive_learning}
\begin{figure*}
    \centering
    \includegraphics[width=0.9\linewidth]{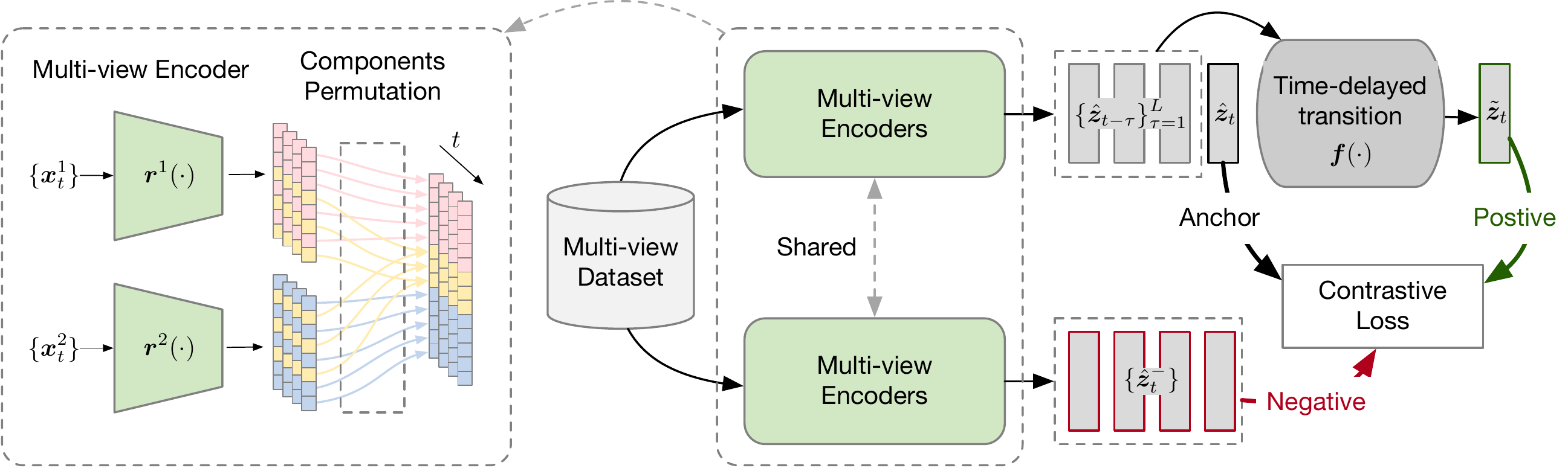}
    \caption{Illustration of our MuLTI framework. 
    On one hand, we recover the view-specific latent factors $\hat{\bm{z}}_t^v$ from individual views, which are then merged to obtain $\hat{\bm{z}}_t$. On the other hand, we exploit the temporal dependency to obtain a causal transited latent factor $\tilde{\bm{z}}_t$ from previously estimated $\hat{\bm{z}}_t$. Thereafter, we regard them as positive pairs to optimize the contrastive loss, which serves as a surrogate of mutual information maximization to achieve identifiability.
    }
    \label{fig:framework}
\end{figure*}

Our goal is to learn an inverse function $r(\cdot)$ from the observed data.
This is achieved by: (1) individually inverting functions $\{r^{v}(\cdot)\}$; (2) a merge function $m:\{\mathcal{Z}^{v}\}\mapsto \mathcal{Z}$, in a way that $r(\cdot)=m(\{r^{v}(\bm{x}^{v})\})$.
Note that each $r^{v}(\cdot)$ corresponds to an inverse of $g^{v}(\cdot)$, and the reverse function $r(\cdot)$ parameterizes the distribution of $q_{r}(\bm{z}_{t}|\{\bm{x}^{v}_{t}\})$.
The merge operator will be introduced in the next section.
The following paragraphs will describe the identification procedure of $\bm{z}_{t}$.
Formally speaking, there exists an indeterminacy mapping $h=r\circ g$ between estimated $\hat{\bm{z}}_{t}=h(\bm{z}_{t})$ and true $\bm{z}_{t}$.
The goal is to identify the true value of $\bm{z}_{t}$.
We exploit the spatiotemporal dependencies of $\bm{z}_{t}$ to infer the values of latent variables from preceding states, as shown in Eq.\eqref{eq:spatial-temporal}.
To achieve this, we introduce the causal transition function $f(\cdot)$ to parameterize the conditional distribution $q_{f,h}(\bm{z}_{t}|\bm{z}_{H_{t}})$, where $\bm{z}_{H_{t}}$ denotes the set of preceding states of $\bm{z}_{t}$.
Empirically, we estimate an alternative value of the complete latent factor by $\tilde{\bm{z}}_{t}=f(\bm{z}_{H_{t}})$.

Motivated by the results of nICA \cite{hyvarinenNonlinearICAUsing2019,zimmermannContrastiveLearningInverts2021}, we can identify the underlying factors of data with a certain low indeterminacy by carefully exploiting the inherent structure of data via a variety of conditional distributions.
To this end, we can achieve model identifiability - that is, the ability to exactly recover the variables $\bm{z}_{t}$ and their associated causal relations - by learning the function $r$, in such a way as to ensure $h$ is an isometry.
Specifically, $h:\mathcal{Z}\mapsto \mathcal{Z}$ must statisfy the condition $\delta(\bm{z}_{t}, \tilde{\bm{z}}_{t})=\delta(h(\bm{z}_{t}),h(\tilde{\bm{z}}_{t}))$ everywhere, where $\delta(\cdot,\cdot)$ is a metric.
Using the conditional dependency in sequential latent states $\{\bm{z}_{t}\}$, a key insight is to estimate the function $f$ for causal relations modeling and estimate reverse function $r$:
\begin{equation}
    \max_{r,f}~\mathcal{I}(\tilde{\bm{z}}_t;\bm{z}_t)=\mathcal{I}(f(\bm{z}_{H_{t}},\bm{\epsilon}_{t});\bm{z}_{t}),
    \label{eq:mutual_information}
\end{equation}
where $\tilde{\bm{z}}_{t}=f(\bm{z}_{H_{t}},\bm{\epsilon}_{t})$ is the estimated latent variable based on the previous latent variables $\bm{z}_{H_{t}}$, $\mathcal{I}(\cdot,\cdot)$ indicates the mutual information.
We denote the prior distribution as $p(\tilde{\bm{z}}_{t} | \bm{z}_{H_{t}}$ and the modeled distribution as $q_{h,f}(\tilde{\bm{z}}_{t}|\bm{z}_{H_{t}})$.
The optimization problem defined in Eq.~\eqref{eq:mutual_information} can also be viewed as a miniization problem for the cross-entropy $H(\cdot\| \cdot)$,
\begin{equation}
    \min_{f,r} H\left(p(\tilde{\bm{z}}_{t}|\bm{z}_{H_{t}})\| q_{f,r}(\tilde{\bm{z}}_{t}|\bm{z}_{H_{t}})\right).
    \label{eq:cond_match}
\end{equation}

Nevertheless, directly optimizing the Eq.~\eqref{eq:cond_match} is typically difficult.
To this end, we employ contrastive learning as a surrogate, given that it has been proven to be a variational bound of mutual information maximization \cite{oordRepresentationLearningContrastive2019}.
Formally, the contrastive loss is defined as follows,
\begin{equation}
\begin{split}
    &\mathcal{L}_{\rm contr}(r,f;\mu,M):= \label{eq:loss_contr}\\
    &\underset{\substack{(\bm{x}, \tilde{\bm{x}}) \sim p_{\text {pos}} \\\left\{\bm{x}_{i}^{-}\right\}_{i=1}^{M} \stackrel{\text { i.i. }}{\sim} p_{\text {data }}}}{\mathbb{E}}\left[-\log \frac{e^{-\delta(\bm{z},\tilde{\bm{z}}) / \mu}}{e^{-\delta(\bm{z},\tilde{\bm{z}}) / \mu}+\sum_{i=1}^{M} e^{-\delta(\bm{z},\bm{z}^{-}) / \mu}}\right],
\end{split}
\end{equation}
$M\in\mathbb{Z}_{+}$ represents the fixed number of negative samples, $p_{\rm data}$ denotes the distribution of all observations, and $p_{\rm pos}$ signifies the distribution of positive pairs, we choose the $\ell_{1}$ as the metric, represented by $\delta(\cdot, \cdot)$, $\mu\geq 0$ represents the temperature. 
Here, we omit the subscript $t$ for simplicity.

As mentioned earlier, we can infer the latent variable $\bm{z}_{t}$ in two ways: directly, with $\hat{\bm{z}}_{t}=r(\{\bm{x}^{v}_{t}\})$, and indirectly, with $\tilde{\bm{z}}_{t}=f(\bm{z}_{H_{t}})$.
Given the expectation that these two inferred variables should be similar, we define the pair $(\hat{\bm{z}}_{t},\tilde{\bm{z}}_{t})$ as a positive pairs.
Subsequently, we sample data from the marginal distribution of observations, recover latent variables $\bm{z}^{-}_{t}$, and form negative pairs with $\hat{\bm{z}}_{t}$. 
At this point, the contrastive learning minimizes the cross-entropy between the ground-truth latent conditional distribution $p(\bm{z}_{t}|\bm{z}_{H_{t}})$ and a specific model distribution $q_{r,h}(\bm{z}_{t}|\bm{z}_{H_{t}})$. 
Thus, we have $\delta(h(\bm{z}_{t}),h(f(\bm{z}_{H_{t}})))\propto\delta(\bm{z}_{t},f(\bm{z}_{H_{t}}))$ for all $\bm{z}_{t}$ and $f(\bm{z}_{H_{t}})$.
This implies that the ground-truth $\bm{z}_t$ lies in the isometric space of the estimated $\hat{\bm{z}}_t$, and can be further recovered through simple transformations.

\subsection{Parametric causal transitions} \label{sec:causal_transition}
To represent the causal transition process of $f(\cdot)$, we consider a widely-used parametric formulation that aligns well with the Granger causality \cite{dingGrangerCausalityBasic2006}.
To parameterize the dependence of $z_{i,t}$ on its causal parents $\textbf{Pa}(z_{i,t})$ in Eq.~\eqref{eq:spatial-temporal}, we model the causal transition as following vector autoregressive (VAR) process.
Let $\bm{A}_{\tau}\in \mathbb{R}^{d\times d}$ be the full rank state transition matrix at lag $\tau$.
Assuming the true $\bm{z}_t$ is known, the causal transition can be represented as follow,
\begin{equation}
    \bm{z}_{t}=f(\bm{z}_{H_{t}},\bm{\epsilon}_{t})=\sum\nolimits^{L}_{\tau=1}\bm{A}_{\tau}\bm{z}_{t-\tau}+\bm{\epsilon}_{t},
    \label{eq:var-transition}
\end{equation}
here, $L$ is max time lag, $\bm{\epsilon}_{t}$ is an additive noise variable.
Using this formulation, we can reformulate the conditional distribution $p(\bm{z}_{t}|\bm{z}_{H_{t}})$ by changing the variables:
\begin{equation}
    p(\bm{z}_{t}|\bm{z}_{H_{t}})=p(\bm{z}_{t}-\sum\nolimits^{L}_{\tau=1}\bm{A}_{\tau}\bm{z}_{t-\tau})=p(\bm{\epsilon}_{t}).
\end{equation}
As the ground-truth $\bm{z}_t$ is not readily available, we reuse the estimated $\hat{\bm{z}}_t$ to calculate Eq.\eqref{eq:var-transition}.
Through this reformulation, we can represent the causal transition using a mutually independent noise distribution.

Through the above proper definitions of the contrastive pairs and the causal transition $f$, minimizers of contrastive loss determine the $h$ up to an isometry:
\begin{theorem}[Minimizers of contrastive objective recover the latent variables and their relations] \label{theorem:ce}
    Let latent space $\mathcal{Z}$ be a convex body in $\mathbb{R}^{d}$, $h=r\circ g:\mathcal{Z}\mapsto\mathcal{Z}$, $f$ be a causal transition function, and $\delta$ be a metric, induced by a norm. If $g$ is differentiable and injective, $r,f$ are expressive enough to be minimizers of contrastive objective $\mathcal{L}_{\rm contr}$ in Eq.~(\ref{eq:loss_contr}) for $M\to +\infty$, then, we have $h=r\circ g$ is invertible and affine mapping and $f$ captures the ground-truth causal relations.
\end{theorem}
Theorem~\ref{theorem:ce} guarantees the identifiability of the encoder $r$ and causal transition function $f$.
Specifically, the contrastive loss enforces $\delta(h(\tilde{\bm{z}}_{t}), h(\bm{z}_{t}))\propto \delta(\tilde{\bm{z}}_{t},\bm{z}_{t})$ almost everywhere, which leads $h$ to be an isometry, \textit{i.e.}, there exists an orthogonal matrix $\bm{U}\in\mathbb{R}^{d\times d}$ such that $\hat{\bm{z}}_{t}=\bm{U}\bm{z}_{t}$ and $\hat{\bm{A}}_{\tau}=\bm{A}_{\tau}\bm{U}^{\top}$.
Further, when $\bm{\epsilon}_{t}$ is non-Gaussian (\textit{e.g.}, Laplacian), and metric $\delta(\cdot,\cdot)$ is non-isotropic (\textit{e.g.}, $\ell_{1}$ norm), there is only a channel permutation $\pi$ between estimations and ground-truth such that $\hat{z}_{i,t}=s_{i}z_{\pi(i),t}$, where $s_{i}$ is a scaling constant.

\subsection{Merge Operation}
As mentioned earlier, a significant challenge in learning a spatiotemporal dependent source from multi-view data lies in aggregating the view-specific source, $\bm{z}^{v}_{t}$, into the complete latent source, $\bm{z}_{t}$.
This is because each view only holds a portion of the total information.
Given that sources estimated from each view exhibit unique uncertainties, ensuring their alignment within the same subspace to yield the desired complete source presents a non-trivial task.
Yet, given that a portion of the information is shared among different sources, it is feasible to establish transformation between them by identifying a common source.
We subsequently cast the problem of identifying a common source as one of optimal transport objectives, which involves searching plans to move factors from view-specific sources to the common source.

\paragraph{Permutation learning}
Assume there is a common source $\bm{c}_{t}\in\mathbb{R}^{d_{c}}$ shared across all view-specific sources $\bm{z}^{v}_{t}$.
While directly searching for the correspondence of common components can be infeasible, we propose to enforce the $\bm{c}^{v}_{t}$ be in the top-$d_{c}$ channels of $\bm{z}^{v}_{t}$ via permutation learning.
Formally, we can achieve this by multiplying the $\bm{z}^{v}_{t}$ with a doubly stochastic matrix $\bm{B}^{v}\in \mathfrak{B}_{d_{v}}$ as follows,
\begin{equation}
    \bar{\bm{z}}^{v}_{t}=\bm{B}^{v}\hat{\bm{z}}^{v}_{t},\quad \bm{c}^{v}_{t}=\bar{\bm{z}}^{v}_{1:d_{c},t}.
    \label{eq:permuation-operation}
\end{equation}
After that, we can impose the permutations that generate embeddings whose top $d_{c}$ entries are most correlated:
\begin{equation}\label{eq:max_cca}
    \max_{\{\bm{B}^{v}\}}\mathbb{E}_{v'\neq v}\text{Tr}\left((\bm{B}^{v}\hat{\bm{z}}^{v}_{t})_{1:d_{c}}(\bm{B}^{v'}\hat{\bm{z}}^{v'}_{t})_{1:d_{c}}^{\top}\right).
\end{equation}
In essence, our goal is to ensure that all extracted instances of $\{\bm{c}^{v}_{t}\}$ have entries closely resembling those of the identified common components.
Notably, the procedure outlined above equates to regularizing these $\{\bm{c}^{v}_{t}\}$ so that they are concentrated around their mean.
To achieve this, we first estimate a mean center from the currently extracted common sources:
\begin{equation}\label{eq:common_center}
    \bm{c}^{*}_{t} = \underset{\bm{c}_{t}\in \mathbb{R}^{d_{c}}}{\arg\min}\sum_{v}\| \bm{c}^{v}_{t}-\bm{c}_{t}\|^{2}.
\end{equation}
In practice, we achieve the objective of Eq.~(\ref{eq:max_cca}) by updating the transport plans $\{\bm{B}^{v}\}$ and minimizing transport cost from top-$d_{c}$ view specific components to $\bm{c}^{*}_{t}$ in an alternating manner:
\begin{equation}\label{eq:opt-permutation}
    \mathcal{L}_{m}=\sum^{2}_{v=1}\mathbb{E}\left[\|(\bm{B}^{v}\bm{z}^{v}_{t})_{1:d_{c}}-\bm{c}^{*}_{t}\|^{2}\right]+\eta R(\bm{B}^{v}),
\end{equation}
where $R(\bm{B}^{v})=-\sum_{i,j}B^{v}_{i,j}\log(B^{v}_{i,j})$, $\eta>0$ is a constant.
This criterion effectively results in learning that $\bm{c}^{1}_{t}=\bm{c}^{2}_{t}$, which represents the ultimate solution for the merging of latent variables.

After this reformulation, with the center $\bm{c}_t^*$ fixed, we can search for the locally optimal $\{\bm{B}_v\}$ separately.
Crucially, each sub-problem now becomes a standard combinatorial assignment problem \cite{peyreComputationalOptimalTransport2020}, which involves the search for optimal transport plans that assign common components from view-specific sources to the common source $\bm{c}_t^*$.
Ultimately, by iteratively applying the center mean estimation procedure from Eq.~\eqref{eq:common_center} and utilizing the Sinkhorn algorithm, we can achieve the minimizer of our final objective, as expressed in Eq.~\eqref{eq:opt-permutation} (refer to Appendix~\ref{app:permuation_expansion} for additional details).
We can consider the application of the transformation $\bm{B}^{v}$ as a process of smooth component sorting.
When we can identify estimated latent sources up to a permutation transformation, we can reduce the doubly stochastic matrices ${\bm{B}^{v}}$ to permutation matrices, \textit{i.e.}, $\{\bm{P}^{v}|\bm{P}^{v}\in \{0,1\}^{d_{v}\times d_{v}}\}$. This resolution addresses the issue of non-corresponding components.

For merge operator $m$, we concanate the $\bm{c}^{*}_{t}$ with all remaining private components:
\begin{equation}
    \hat{\bm{z}}_{t}=\left[\begin{array}{c}
        \bm{c}^{*}_{t} \\
        \bar{\bm{z}}^{1}_{d_{c}+1:d_{1},t} \\
        \bar{\bm{z}}^{2}_{d_{c}+1:d_{2},t}
        \end{array}\right]=m\left(\left[\begin{array}{c}
        \hat{\bm{z}}^{1}_{t} \\
        \hat{\bm{z}}^{2}_{t}
        \end{array}\right]\right).
    \label{eq:merge-operator}
\end{equation}
Consequently, we can input the merged latent source into the contrastive learning and the inference of causal relations, as outlined in Sec~\ref{sec:contrastive_learning}, thereby completing the entire objective.

\begin{theorem}[Minimizers of the multi-view objective maintains the channel corespondency] \label{theorem:mult-view}
    Let $\mathcal{Z}^{v}\subseteq \mathbb{R}^{d_{v}},\mathcal{Z}=\cup\mathcal{Z}^{v}$ and $\cap \mathcal{Z}^{v}\neq \emptyset$. If $\bm{B}^{v}\in \mathfrak{B}_{d_{v}}$ be a doubly stochastic matrix, $r,f,\{\bm{B}^{v}\}$ are the minimizers of contrastive objective for $M\to +\infty$, then, we have $h=r\circ g$ is an isometry, and $\{\bm{B}^{v}\}$ can rearrange the components of each view source such that common components from each view source are aligned.
\end{theorem}

\section{Optimization}
We use deep neural networks to implement $f(\cdot)$ and $r^{v}(\cdot)$.
The process of learning latent variables $\bm{z}_{t}$ from multi-view data involves contrastive learning on estimated latent variables and the learning of permutations for view-specific variables.
We jointly train view-specific reverse networks as well as a causal transition network.
The permutation matrices $\{\bm{B}^v\}$ are also alternatively optimized during this process.

\paragraph{Enhance noises distribution learning}
Contrastive learning is designed to minimize the distance between positive pairs, in other words, it brings $\hat{\bm{z}}_{t}$ and $f(\hat{\bm{z}}_{H_{t}})$ closer together.
To emphasize this property, we further employ an objective to minimize the residual as:
\begin{equation}
    \mathcal{L}_{\epsilon}=\mathbb{E}[\delta(\hat{\bm{z}}_{t},f(\hat{\bm{z}}_{H_{t}}))].
\end{equation}
In parallel, to enhance the mutual independence property of noise $\epsilon_{i,t}$, we follow the approach of \cite{yaoLearningTemporallyCausal2021} and create a discriminator network $\mathcal{D}(\cdot)$ implemented by MLPs.
This network is used for discriminating the $\{\hat{\epsilon}_{i,t}\}$ and randomly permutated versions $\{\hat{\epsilon}_{i,t}\}^{\text{perm}}$,
\begin{equation}
    \mathcal{L}_{\mathcal{D}}=\mathbb{E}[\log \mathcal{D}(\{\hat{\epsilon}_{i,t}\}) - \log (1-\mathcal{D}(\{\hat{\epsilon}_{i,t}\}))].
\end{equation}
Combining all these elements with weights, the objective is:
\begin{equation}
    \mathcal{L}=\mathcal{L}_{\text{contr}} + \beta_{1} \mathcal{L}_{m} + \beta_{2} \mathcal{L}_{\epsilon} + \beta_{3} \mathcal{L}_{\mathcal{D}}.
\end{equation}
For stable training, we perform the optimization of $\mathcal{L}_{\text{contr}}$ and $\mathcal{L}_{m}$ alternately.

\section{Experiments}
In this section, we present our experimental results on three multi-view scenarios to validate the superiority of MuLTI.
More empirical results can be found in Appendix~\ref{app:experiments}.

\subsection{Datasets}
We use synthetic and real-world datasets for evaluating the latent process identification task.
Different view-specific encoders are used in different datasets for extracting latent variables.
In what follows, we briefly introduce the datasets.

\noindent \textbf{Multi-view VAR} is a synthetic dataset modified from \cite{yaoLearningTemporallyCausal2021}.
To generate the latent process, we first select transition matrices $\{\bm{A}_{\tau}\}^{L}_{\tau=1}$ from a uniform distribution $U[-0.5, 0.5]$, and then sample initial states $\{\bm{z}_{t}\}^{L}_{t=1}$ from a normal distribution to generate sequences $\{\bm{z}_{t}\}^{T}_{t=1}$.
To generate multi-view time series, we randomly select $d_{1}$ and $d_{2}$ dimensions of latent variables for each view.
The nonlinear observations of each view are created by mixing latent variables with randomly parameterized MLPs $g^{v}(\cdot)$ following the previous work \cite{hyvarinenNonlinearICATemporally2017}.

\noindent \textbf{Mass-spring system} is a video dataset adopted from \cite{liCausalDiscoveryPhysical2020} specifically designed for a multi-view setting.
A mass-spring system consists of 8 movable balls; the balls are either connected by springs or are not connected at all.
As a result of random external forces, the system exhibits different dynamic characteristics depending on its current state and internal constraints.
We create two views for this video dataset, as shown in Figure~\ref{fig:balls-illustration}, in each of which only 5 of the 8 balls are observable.
Consequently, each pair of views includes two videos, each with a duration of 80 frames.
This dataset contains 5,000 pairs of video clips in total.
\begin{figure}[ht]
    \centering
    \includegraphics[width=0.8\linewidth]{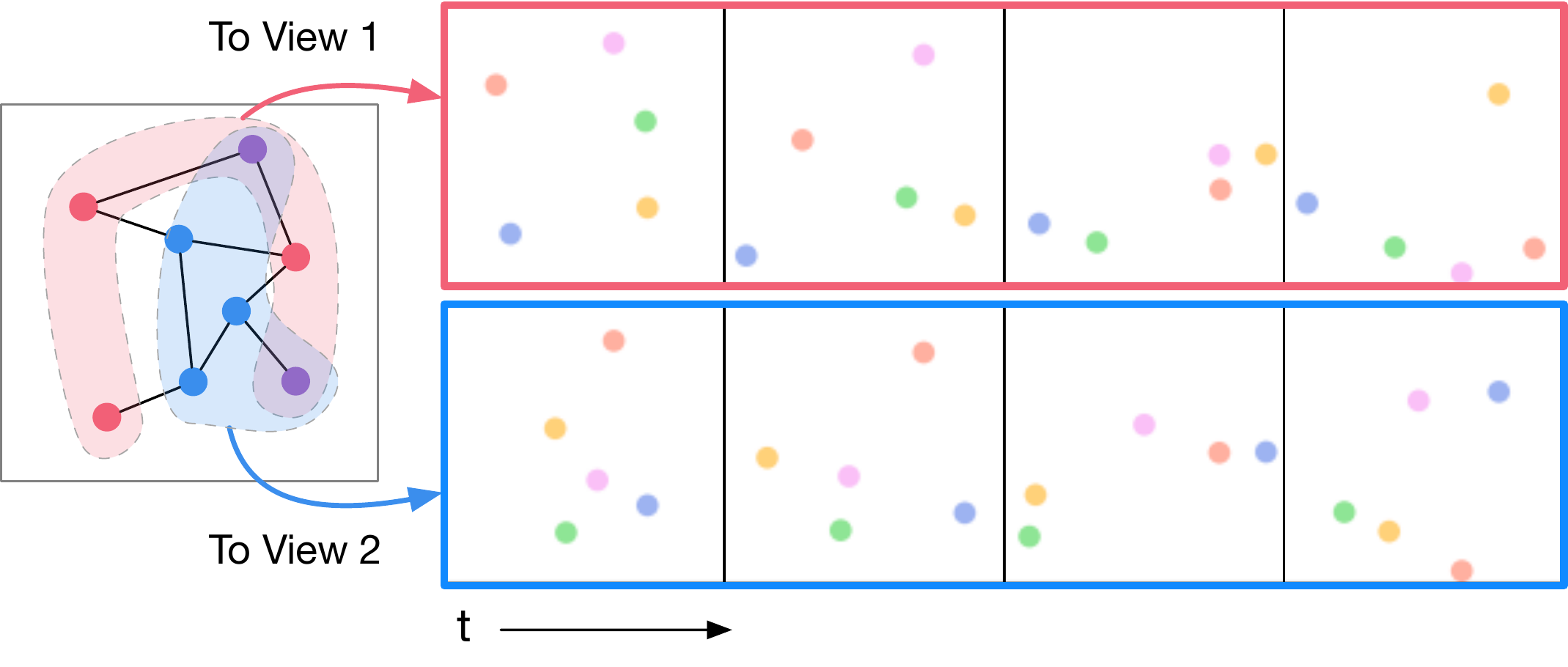}
    \caption{Illustration of multi-view Mass-spring system.}
    \label{fig:balls-illustration}
\end{figure}

\noindent \textbf{Multi-view UCI} Daily and Sports Activities \cite{altunComparativeStudyClassifying2010} is a multivariate time series dataset comprising 9,120 sequences, capturing sensor data for 19 different human actions performed by 8 subjects.
Each sample activity contains 125 time steps, each captured by nine sensors, yielding a total of 45 dimensions.
We construct the multi-view setting following the \cite{liMultiViewTimeSeries2016}.
The 27 dimensions located on the torso, right arm, and left arm are collectively regarded as view 1, while the remaining 18 dimensions on the left and right legs constitute view 2.

\subsection{Evaluation Setup}
To measure the identifiability of latent causal variables, we compute Mean Correlation Coefficient (MCC) on the validation datasets for revealing the indeterminacy up to permutation transformations, and $R^{2}$ for revealing the indeterminacy up to linear transformations.
We assess the accuracy of our causal relations estimations by comparing them with the actual data structure, quantified via the Structural Hamming Distance (SHD) on the validation datasets.

\paragraph{Baselines}
Four different methods for estimating latent variables: (1) BetaVAE \cite{higgins2016beta} neither accounts for the temporal structure nor provides an identifiability guarantee; (2) SlowVAE \cite{klindtNonlinearDisentanglementNatural2020}, PCL \cite{hyvarinenNonlinearICATemporally2017}, and GCL \cite{hyvarinenNonlinearICAUsing2019} identify independent sources from observations; (3) ShICA \cite{richardSharedIndependentComponent2021} identifies the shared sources from multi-view observations; (4) CL-ICA \cite{zimmermannContrastiveLearningInverts2021} recovers the latent variables via contrastive learning; (5) LEAP \cite{yaoLearningTemporallyCausal2021} utilizes temporally engaged sources and causal mechanism but only suitable single view data.

For multi-view time series adaptation of baseline methods -- $\beta$-VAE, SlowVAE, PCL, CL-ICA, and LEAP -- we concatenate the series into single-view format, aligning them by the feature dimension.
In the GCL setting, we designate as positive pairs the views originating from the same source, and negative pairs those stemming from distinct sources.

\paragraph{Implementation details}
For the Multi-view VAR dataset, we use architectures composed of MLPs and Leaky ReLU units.
Ground-truth latent variable dimension $d$ is set to 10, the noise distribution is set to Laplacian$(0, 0.05)$.
We set the batch size to 2400, employ the Adam optimizer with a learning rate of 0.001, and utilize $\beta_{1}=0.01, \beta_{2}=0.01, \beta_{3}=1e-5$.
To verify the influence of the overlapping ratio of views, we select $d_{c}\in \{10, 4, 2, 0\}$ for evaluation.
Here, $d_{c}=0$ and $d_{c}=10$ indicate situations where the latent variables for generating views are fully overlapped and have no overlaps, respectively.
To ensure a fair comparison, all baseline methods utilize similar encoders, and the time-lag $L$ of the causal transition module is set to equal the ground truth.

For the Mass-spring system, we create view-specific encoders following \cite{liCausalDiscoveryPhysical2020}, which has similar architectures to the unsupervised keypoints discovery perception module.
We set the time lag $L = 2$ for the causal transition module as the approximated process corresponds to a mass-spring system, which is a dynamic system of second order.
The dimension of estimated latent variables is set to be the same as ground-truth, \textit{i.e.}, 8 coordinates $(x,y)$ account for $d=16$.
In practice, we first pre-train two pairs of keypoint encoder-decoder in each view.
Then, the pre-trained encoders are taken as $r^{v}(\cdot)$ corresponding to each view.

For the Multi-view UCI dataset, we construct two separate encoders with MLPs for two views.
The latent dimensions of view 1 and view 2 are set to $d_{1}=12$ and $d_{2}=9$ respectively, the shared latent dimension is set to $d_{c}=3$, and the complete latent dimension accounts for $d=18$.
The time lag of the causal transition module is set to $L=1$.
We initially train the MuLTI on the complete dataset, following which we extract the latent variables for downstream tasks.

\subsection{Main Results}
\begin{table*}[tbp]
\centering
\caption{Identifiability results on the VAR process ($d=10,L=2$). Mean+standard deviation over 5 random seeds.}
\label{tab:var-results}
\begin{tabular}{p{1.7cm}<{\centering}|llllllll}
\toprule
\multirow{2}{*}{Methods} & \multicolumn{4}{c|}{$R^2~(\%)$}                                                                                                        & \multicolumn{4}{c}{MCC}                                                                                                       \\ \cmidrule(l){2-9} 
                         & \multicolumn{1}{c}{$d_{c}=d$} & \multicolumn{1}{c}{$d_{c}=4$} & \multicolumn{1}{c}{$d_{c}=2$} & \multicolumn{1}{c|}{$d_{c}=0$} & \multicolumn{1}{c}{$d_{c}=d$} & \multicolumn{1}{c}{$d_{c}=4$} & \multicolumn{1}{c}{$d_{c}=2$} & \multicolumn{1}{c}{$d_{c}=0$} \\ \midrule
$\beta$-VAE              & 40.03{\tiny $\pm$5.31}        & 41.43{\tiny $\pm$4.22}        & 40.67{\tiny $\pm$3.44}        & 38.36{\tiny $\pm$2.71}         & 29.37{\tiny $\pm$8.20}        & 38.31{\tiny $\pm$6.31}        & 31.92{\tiny $\pm$5.61}        & 41.88{\tiny $\pm$4.62}        \\
SlowVAE                  & 60.32{\tiny $\pm$ 5.56} & 63.21{\tiny $\pm$ 5.51} & 62.12{\tiny $\pm$ 6.13} & 61.23{\tiny $\pm$ 3.17} & 50.32{\tiny $\pm$ 3.49}     & 51.24{\tiny $\pm$ 4.71} & 52.55{\tiny $\pm$ 2.39} & 54.21{\tiny $\pm$ 3.07} \\
PCL                      & 69.78{\tiny $\pm$3.20}        & 80.24{\tiny $\pm$3.32}        & 75.71{\tiny $\pm$2.71}        & 74.56{\tiny $\pm$2.90}         & 52.64{\tiny $\pm$2.33}        & 54.63{\tiny $\pm$2.12}        & 53.42{\tiny $\pm$1.91}        & 55.61{\tiny $\pm$2.62}        \\
GCL                      & 73.45{\tiny $\pm$4.34}        & 82.41{\tiny $\pm$3.12}        & 77.65{\tiny $\pm$2.52}        & 73.21{\tiny $\pm$3.31}         & 53.55{\tiny $\pm$1.71}        & 57.32{\tiny $\pm$2.41}        & 52.42{\tiny $\pm$1.70}        & 55.23{\tiny $\pm$2.31}        \\
ShICA                    & 41.08{\tiny $\pm$5.76}        & 39.65{\tiny $\pm$4.51}        & 37.61{\tiny $\pm$4.32}        & 38.43{\tiny $\pm$3.12}         & 28.71{\tiny $\pm$5.21}        & 37.21{\tiny $\pm$5.13}        & 29.31{\tiny $\pm$4.97}        & 31.47{\tiny $\pm$3.61}        \\
CL-ICA                   & 21.48{\tiny $\pm$6.71}        & 29.37{\tiny $\pm$4.32}        & 24.27{\tiny $\pm$3.33}        & 23.89{\tiny $\pm$2.71}         & 26.79{\tiny $\pm$5.51}        & 33.35{\tiny $\pm$4.51}        & 29.39{\tiny $\pm$5.11}        & 23.68{\tiny $\pm$4.90}        \\
LEAP                     & \textbf{99.59}{\tiny $\pm$0.05}        & 99.67{\tiny $\pm$0.09}        & 99.48{\tiny $\pm$0.07}        & \textbf{99.74}{\tiny $\pm$0.06}         & 57.56{\tiny $\pm$2.44}        & 62.48{\tiny $\pm$2.12}        & 65.44{\tiny $\pm$1.61}        & 67.75{\tiny $\pm$1.32}        \\
\rowcolor{Gray} MuLTI                    & 99.46{\tiny $\pm$0.09}        & \textbf{99.72}{\tiny $\pm$0.05}        & \textbf{99.68}{\tiny $\pm$0.07}        & 99.45{\tiny $\pm$0.08}         & \textbf{99.80}{\tiny $\pm$0.12}        & \textbf{99.27}{\tiny $\pm$0.43}        & \textbf{99.39}{\tiny $\pm$0.09}        & \textbf{99.43}{\tiny $\pm$0.04}        \\ \bottomrule
\end{tabular}
\end{table*}

\paragraph{MuLTI achieves the best identifiability}
We report both $R^{2}$ and MCC in Table~\ref{tab:var-results} for revealing the identifiability of each method.
As shown in the results, MuLTI achieves the best identifiability across most settings, particularly over the MCC metrics, where MuLTI outperforms all rivals by a notable margin.
The results indicate that, without modeling the causal transition in the conditional distribution, the independence conditions of methods such as $\beta$-VAE, SlowVAE, ShICA, CL-ICA, PCL, and GCL fail to recover the latent variables, even up to linear transformations.
LEAP successfully identifies the latent process up to linear transformations using causal transition constraints.
Yet, it struggles to recover the process up to permutation transformation for observations from distinct views, despite conditions suitable for permutation identifiability.

\paragraph{MuLTI successfully identifies the latent process}
\begin{figure}[t]
    \centering
    \includegraphics[width=0.85\linewidth]{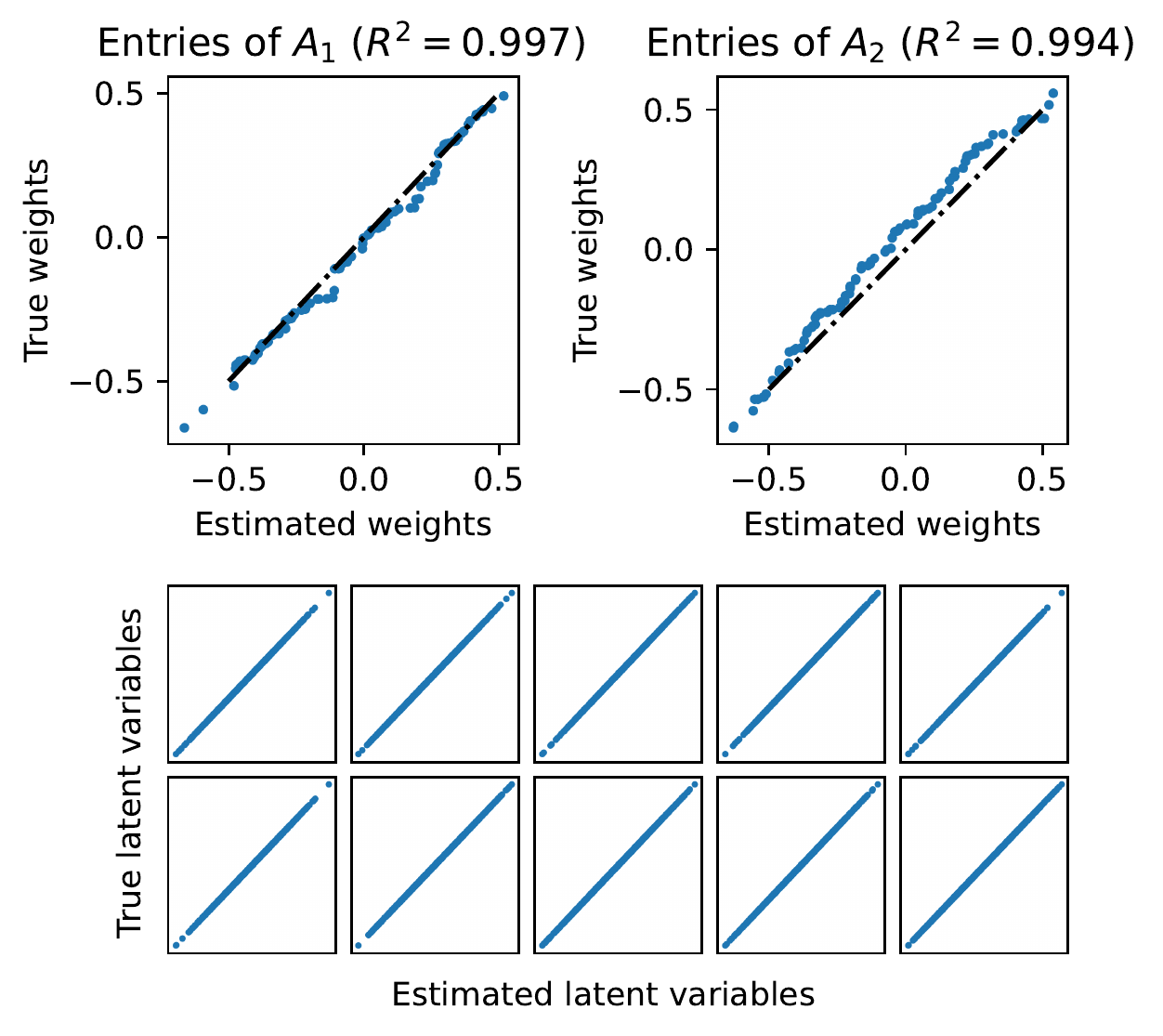}
    \caption{Results for latent VAR process ($d=10,d_{c}=2$) identification. Scatters are paired estimations and ground-truth.}
    \label{fig:var-recover}
\end{figure}
We show that MuLTI successfully identifies the latent process on both the Multi-view VAR dataset and the Mass-spring system.
Figure~\ref{fig:var-recover} demonstrates the performance of our MuLTI in estimating the latent process where $d=10,d_{c}=2$.
The results suggest that our model nearly perfects the learning of the VAR transition matrices, with the identified latent variables closely resembling the true variables, thus implying a comprehensive identification of the latent process.

The objective of the Mass-spring system task is to estimate latent variables, representing the coordinates of each ball, and transition matrices that elucidate their connections.
As displayed in the left panel of Figure~\ref{fig:exp-ball-corr}, the estimated latent variables precisely match the ground-truth, and the recovered transition matrices correspond to the ground-truth ball connections, as indicated by an $\text{SHD}=0$.
\begin{figure}[t]
    \centering
    \includegraphics[width=0.8\linewidth]{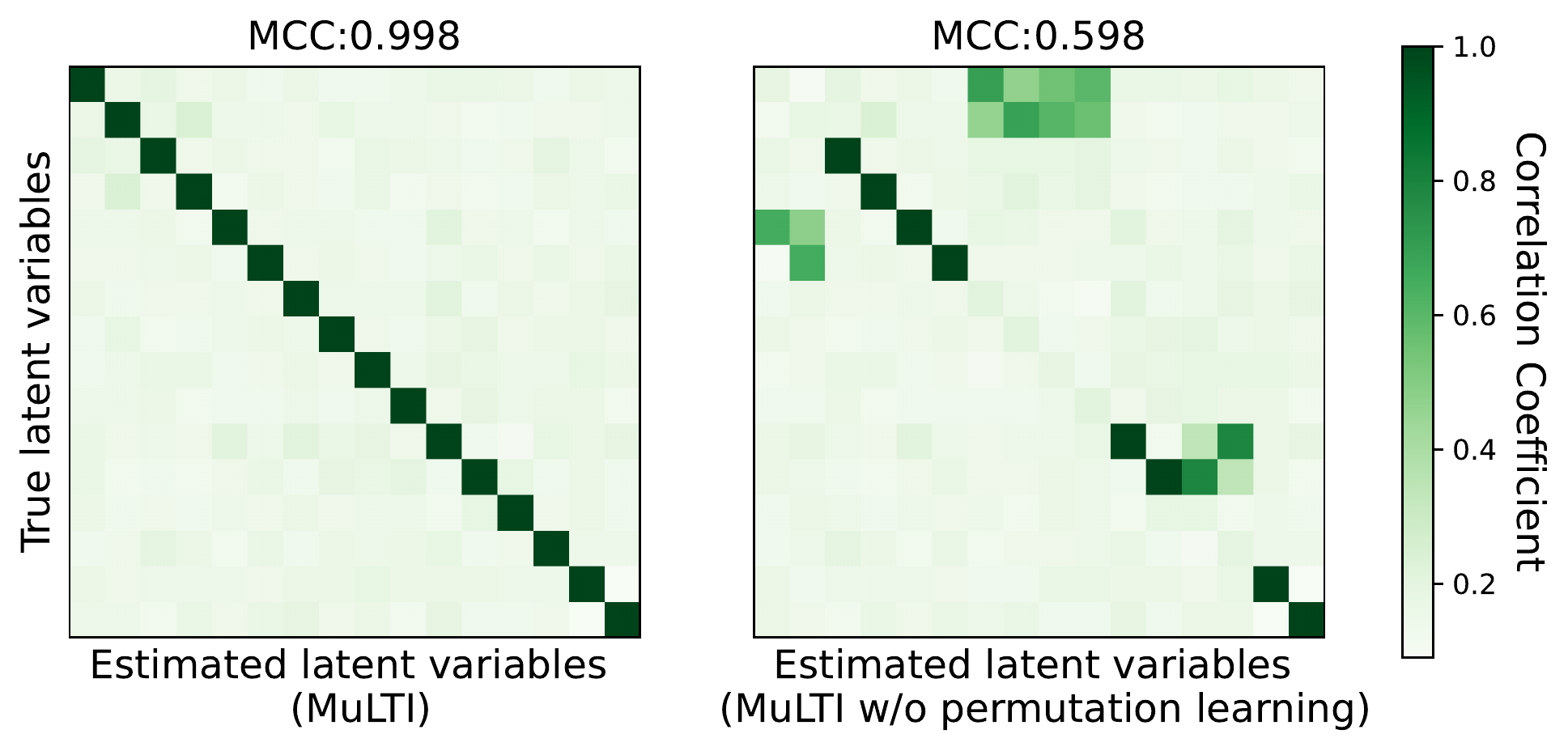}
    \caption{Correlation matrix of recovered latent variables vs. ground-truth latent variables of Mass-spring system.}
    \label{fig:exp-ball-corr}
\end{figure}

\paragraph{MuLTI improves downstream task}
Given that causal variables are often not readily available, directly assessing their identification and relationships in real-world data proves challenging.
However, under most conditions, identifying the latent process from observations can simplify downstream tasks (\textit{e.g.}, classification and clustering) or directly aid in the extraction of meaningful features.
Therefore, we employ a real-world multi-view time series dataset to demonstrate the effectiveness of latent process identification.
We employ multivariate time series classifiers - specifically, reservoir computing \cite{bianchiReservoirComputingApproaches2021}.
This approach encodes the multivariate time series data into a vectorial representation in an unsupervised manner, allowing us to evaluate learned latent variables $\{\hat{\bm{z}}_{t}\}^{T}_{t=1}$ in the downstream task.

To demonstrate the efficacy of the learned latent representations, we independently train both the baseline methods and MuLTI on the entire dataset, subsequently extracting 18-dimensional features to form a new dataset.
Then, we train and evaluate the RC classifier on the new datasets extracted by each method.
The results are shown in Table~\ref{tab:uci-results}. 
We report classification accuracy with varying training set sizes, \textit{i.e.}, we randomly select $N_{tr} \in \{ 10, 20, 30, 40, 50 \}$ samples per class for training, using the remaining samples for testing.
Compared to the baseline methods, the pre-trained features from MuLTI enhance the downstream performance of the RC classifier, particularly for small training set classification tasks.
\begin{table}[tbp]
\centering
\caption{Classification accuracy on Multi-view UCI dataset. $N_{tr}$ is the number of training samples randomly chosen from each class.}
\label{tab:uci-results}
\resizebox{\columnwidth}{!}{%
\begin{tabular}{@{}llllll@{}}
\toprule
$N_{tr}$ & raw data          & +$\beta$-VAE          & +LEAP                 & +ShICA                & +MuLTI                \\ \midrule
10       & 74.21{\tiny $\pm$0.26} & 72.53{\tiny$\pm$1.13} & 73.57{\tiny$\pm$1.90} & 73.17{\tiny$\pm$0.77} & \textbf{85.59}{\tiny$\pm$0.05} \\
20       & 90.00{\tiny$\pm$0.38}  & 85.82{\tiny$\pm$0.53} & 86.10{\tiny$\pm$0.86} & 86.59{\tiny$\pm$0.51} & \textbf{92.14}{\tiny$\pm$0.38} \\
30       & 92.92{\tiny$\pm$0.74}  & 89.37{\tiny$\pm$1.54} & 88.90{\tiny$\pm$0.69} & 91.29{\tiny$\pm$0.31} & \textbf{94.18}{\tiny$\pm$0.14} \\
40       & 94.95{\tiny$\pm$0.16}  & 91.16{\tiny$\pm$0.49} & 91.79{\tiny$\pm$0.07} & 94.01{\tiny$\pm$0.55} & \textbf{96.15}{\tiny$\pm$0.81} \\
50       & 94.92{\tiny$\pm$0.28}  & 92.53{\tiny$\pm$0.04} & 93.00{\tiny$\pm$0.03} & 94.78{\tiny$\pm$0.04} & \textbf{96.82}{\tiny$\pm$0.30} \\ \bottomrule
\end{tabular}%
}
\end{table}

\subsection{Ablation Studies}
\begin{table}[tbp]
\centering
\caption{Ablation results on VAR dataset ($d=10,d_{c}=4,L=2$).}
\label{tab:ablation-var}
\begin{tabular}{@{}lll@{}}
\toprule
\multicolumn{1}{c}{MuLTI} & \multicolumn{1}{c}{$R^2~(\%)$} & MCC            \\ \midrule
w/o Causal transition     & 44.67{\tiny$\pm$2.31}            & 33.47{\tiny$\pm$7.60} \\
w/o Permutation learning      & 99.32{\tiny$\pm$0.07}            & 77.45{\tiny$\pm$0.47} \\
w/o Residuals minimization   & 99.48{\tiny$\pm$0.04}            & 82.32{\tiny$\pm$2.63} \\
w/o Noise discriminator      & 99.23{\tiny$\pm$0.05}            & 98.32{\tiny$\pm$0.27} \\ \bottomrule
\end{tabular}
\end{table}

\paragraph{Module ablations}
To verify the effectiveness of each module within MuLTI, we conduct ablation studies by individually removing each module.
The results tested on the Multi-view VAR dataset are shown in Table~\ref{tab:ablation-var}.
When the causal transition module is removed, positive pairs are formed from temporally adjacent states, \textit{i.e.}, $(\hat{\bm{z}}_{t},\hat{\bm{z}}_{t+1})$.
In the setting of removing permutation learning, the view-specific variables are simply concatenated along the feature dimension and mapped with a linear layer to match the dimension of the transition function.
Results indicate the crucial role of the causal transition function in identifying the spatial-temporal dependent latent process.
As shown, permutation learning also greatly contributes to the identifiability under this multi-view setting.
Furthermore, Figure~\ref{fig:exp-ball-corr} from the Mass-spring system experiment reveals that, without permutation learning, only partial identification of the latent variables is possible.

\paragraph{Effect of noise and distance metrics}
To evaluate the effect of noise distribution type and distance metric $\delta(\cdot,\cdot)$, we conduct experiments under different combinations. The results shown in Table~\ref{tab:metric-ablation} indicate that the Laplacian noise condition accompanied with $\ell_{1}$ norm is key to identifying the ground-truth latent processes up to permutation equivalence.
Nevertheless, in most conditions, we can still identify the latent process up to linear equivalence.
\begin{table}[tbp]
\centering
\caption{Identifiability with different noise prior \& metric. Note that the \emph{Identity} indicates the untrained model, and \emph{Supervised} indicates the model trained with MSE between ground-truth and estimated latent variables. Mean+standard deviation over 5 random seeds.}
\label{tab:metric-ablation}
\begin{tabular}{@{}lccc@{}}
\toprule
Noise   type             & $\delta(\cdot,\cdot)$ & $R^{2} ~(\%)$        & MCC            \\ \midrule
Identity                 & /                     & 66.12{\tiny$\pm$2.77} & 43.91{\tiny$\pm$2.56} \\
Supervised               & /                     & 99.78{\tiny$\pm$0.07} & 99.92{\tiny$\pm$0.02} \\ \midrule
Normal ($\sigma=0.1$)    & $\ell_{2}$                    & 99.72{\tiny$\pm$0.03} & 67.21{\tiny$\pm$1.41} \\
Laplace ($\lambda=0.1$)  & $\ell_{2}$                    & 99.71{\tiny$\pm$0.02} & 66.74{\tiny$\pm$0.21} \\
Normal ($\sigma=0.1$)    & $\ell_{1}$                    & 99.69{\tiny$\pm$0.04} & 70.52{\tiny$\pm$0.27} \\
Laplace ($\lambda=0.1$)  & $\ell_{1}$                    & 99.74{\tiny$\pm$0.02} & 99.31{\tiny$\pm$0.07} \\
Laplace ($\lambda=0.05$) & $\ell_{1}$                    & 99.77{\tiny$\pm$0.03} & 99.81{\tiny$\pm$0.02} \\ \bottomrule
\end{tabular}
\end{table}

\section{Related Works}

\paragraph{Multi-view representation learning}
Many realistic conditions involve data from multiple channels and domains; these data must be analyzed together in order to obtain a reliable estimation of latent relations, \textit{e.g.} neurophysiology and behavior time series data \cite{uraiLargescaleNeuralRecordings2022}.
Typical methods of learning share representations from multi-view data include Canonical Correlation Analysis (CCA) \cite{hotellingRelationsTwoSets1992} and its variants \cite{akaho2006kernel,andrew2013deep,laiKernelNonlinearCanonical2000,wangDeepMultiViewRepresentation,zhao2017multi}, but most of these methods only consider learning the common source of each view.
More recently, some improvements are done in dividing the sources into common and private components \cite{lyu2021understanding,huang2018multimodal,hwang2020variational,luoDifferentiableHierarchicalOptimal2022}.
However, it is not the case when there are relations (causal influences) between source components, due to most of them relying on independent source assumptions \cite{greseleIncompleteRosettaStone2020}.
In contrast, our work seeks to recover latent variables when there are causal relations between them.

\paragraph{Identifiability}
Identifiability means the estimated latent variables are exactly equivalent to the underlying ones or at least up to simple transformation. Traditional independent component analysis (ICA) \cite{belouchrani1997blind} assumes sources are linear mixed, but such linearity relation may not hold in many applications. To overcome this limitation, the non-linear ICA (nICA) \cite{hyvarinenNonlinearIndependentComponent1999} has attracted great attention.
A variety of practical methods \cite{hyvarinenNonlinearICATemporally2017,hyvarinenNonlinearICAUsing2019,khemakhemVariationalAutoencodersNonlinear2020,sorrensonDisentanglementNonlinearICA2022} have been proposed to achieve identifiability through discriminating data structures, such as manually constructing positive and negative pairs and training models to produce features to distinguish whether input data are distorted, such as temporal structure \cite{DBLP:conf/nips/HyvarinenM16} and randomized distorted data pairs \cite{hyvarinenNonlinearICAUsing2019}. 
More recently, \citeauthor{greseleIncompleteRosettaStone2020} adopt similar practices into multi-view data, but still limited in source independent assumptions. \citeauthor{yaoLearningTemporallyCausal2021} make attempts on dependent sources but their method can only handle single-view data.
In this work, we explore a general case that identifies dependent processes from multi-view temporal data.

\section{Conclusion}
In this work, we make the first attempt toward a general setting of the multi-view latent process identification problem and propose a novel MuLTI framework based on contrastive learning. The key idea is to identify the latent variables and dependent structure through the construction of conditional distribution and aggregating variables across views.
Empirical results suggest that our methods can successfully identify the spatial-temporal dependent latent process from multi-view time series observations.
We hope our work can raise more attention to the importance of identifying and aggregating latent factors in understanding multi-view data.

\section{Acknowledgments}
This work is supported by the National Key R\&D Program of China (2020YFB1313501), National Natural Science Foundation of China (T2293723, 61972347), Zhejiang Provincial Natural Science Foundation (LR19F020005), the Key R\&D program of Zhejiang Province (2021C03003, 2022C01119, 2022C01022), the Fundamental Research Funds for the Central Universities (No. 226-2022-00051).
JZ also want to thank the support by the NSFC Grants (No. 62206247) and the Fundamental Research Funds for the Central Universities.

\bibliographystyle{named}
\bibliography{ijcai23}

\clearpage
\appendix
\counterwithin{theorem}{section}
\counterwithin{table}{section}
\counterwithin{figure}{section}

\section{Notations and Terminology}
\label{app:notations}

\begin{table}[h!]
\centering
\caption{List of notations.}
\begin{tabular}{p{0.1\linewidth}p{0.8\linewidth}}
\toprule
\multicolumn{2}{c}{\textbf{Index}} \\
$i,j$ & Variable element (vector/matrix element) index \\
$t$   & Time step \\
$\tau$ & Time delay step \\
$v$    & View index \\
$H_{t}$ & A set of indices represent the previous time steps back from the $t$-step \\
\multicolumn{2}{c}{\textbf{Variable}} \\
$\bm{x}^{v}_{t}$ & Observations of $v$-th view \\
$\bm{z}_{t}$ & Ground-truth latent source variable at $t$-step \\
$\bm{z}^{v}_{t}$ & Ground-truth latent source variable of $v$-th view \\
$\epsilon_{t,i}$ & Mutually independent exogenous noise variable in generating $z_{t,i}$ \\
$\hat{\bm{z}}_{t}$ & Reconstructed factors by $r(\cdot)$ \\
$\tilde{\bm{z}}_{t}$ & Positive pair of $\bm{z}_{t}$ \\
$\textbf{Pa}(\cdot)$ & Set of direct cause nodes/parents of variable $z_{t,i}$ \\
$\bm{A}_{\tau}$ & Transition matrix with $\tau$ steps delay \\
$\mathfrak{B}_{d}$ & $d$-dimensions Brikoff polytope \\
$\bm{B}^{v}$  & $\bm{B}^{v}\in \mathfrak{B}_{d_{v}}$ is the doubly stochastic matrix with $d_{v}$-dimensions \\
$\mu$ & Temperature constant for contrastive learning \\
$\eta$ & Doubly stochastic matrix entropy penalty scale factor \\

\multicolumn{2}{c}{\textbf{Function and Hyper-parameter}} \\
$g$ & Potential invertible function which control the $v_{i}$ generating \\
$r$ & Function which encodes data to latent factors \\
$h$ & $h=r\circ g:\mathcal{Z}\mapsto \mathcal{Z}$ the transformation between ground-truth and estimated factors \\
$f$ & Causal transition function which encodes time-delayed causal relations between factors \\

\bottomrule
\end{tabular}
\end{table}

\section{Permutation Learning for the Merge Operation}
\label{app:permuation_expansion}

In this section, we briefly introduce the optimization of permutation learning.
As indeterminacy exists in encoded view-specific variables, permutation learning is a preliminary step towards merging operations by searching for the order correspondence between components across views.

Recall that we adopt an alternative optimization strategy to optimize the overall objective. Specifically, to perform permutation learning, we freeze other modules of the model to optimize the permutation matrices and the resultant objective is as follows:

\begin{equation*}
    \min_{\bm{B}^{v}}\mathbb{E}\left[\|(\bm{B}^{v}\bm{z}^{v}_{t})_{1:d_{c}}-\bm{c}^{*}_{t}\|^{2}\right] + \eta R(\bm{B}^{v}),
    \label{eq:supp-merge-obj}
\end{equation*}
where $R(\bm{B}^{v})=-\sum_{i,j} B_{i,j}\log B_{i,j}$, and $\bm{c}^{*}_{t}$ can be regarded as a mean variable of all shared components.
Technically, we keep a set of trainable matrices $\{\bm{B}^{v}\}$ that correspond to the dimension permutations for each view-specific encoder.
At each permutation matrix optimization step, MuLTI searches for top-$d_{c}$ most relevant components between view-specific latent variables.
Note that if the component is correctly ordered, each shared components are able to match each other, which leads to a low variance in this dimension. Hence, optimizing Eq.~(\ref{eq:max_cca}) induces a set of proper permutation matrices for merging the latent factors from different views. 

The approximate permutation matrices $\{\bm{B}^{v}\}$ should follow some constraints in order to prevent the degradation of permutation learning, namely, to keep every $\bm{B}^{v}$ within the Birkhoff Polytope at the whole procedure.
Let $\mathcal{B}_{d}$ denote the $d$-Birkhoff polytope, i.e., the set of doubly stochastic matrices of dimension $d$, and $\mathcal{P}_{d}$ denote the set of permutation matrices of dimension $d$,
\begin{align*}
    \mathcal{B}_{d}=\{\bm{B}\in [0, 1]^{d\times d} ,\bm{B}\bm{1}_{d}=\bm{1}_{d},\bm{B}^{\top}\bm{1}_{d}=\bm{1}_{d}\},\\
    \mathcal{P}_{d}=\{\bm{B}\in \{0, 1\}^{d\times d},\bm{B}\bm{1}_{d}=\bm{1}_{d},\bm{B}^{\top}\bm{1}_{d}=\bm{1}_{d} \},
\end{align*}
where $\mathcal{P}_{d}$ indicates rigorous permutation, and $\mathcal{B}_{d}$ indicates the relaxed version of the former. While our ultimate goal is to obtain a binary permutation matrix in $\mathcal{P}_{d}$, the objective can be hard to be resolved due to the discrete constraints. In practice, we search for a relaxed optimizer in $\mathcal{B}_{d}$ to optimize Eq.(15) and finally binarize the obtained matrix to obtain the desired solution. 

We decompose the optimization procedure into two steps. First, the standard stochastic gradient algorithm is applied to optimize Eq.(15), where $\bm{B}^v$s are searched in the $\mathbb{R}^n$ space without considering the constraints. Next, denote the obtained matrix by $\bm{M}$, we introduce the Sinkhorn-Knopp method to transform it to a doubly stochastic matrix in $\mathcal{B}_d$.
Let $S(\cdot)$ denote the Sinkhorn operator. The updating rule of the Sinkhorn-Knopp algorithm is as follows,

\begin{align*}
    S^{0}(\bm{M})&=\exp(\bm{M}),\\
    S^{l}(\bm{M})&=\mathcal{T}_{\text{col}}\left(\mathcal{T}_{\text{row}}(S^{l-1}(\bm{M}))\right),\\
    S(\bm{M})&=\lim_{l\rightarrow \infty}S^{l}(\bm{M}),
\end{align*}
where $\mathcal{T}_{\text{row}}(\bm{M})=\bm{M}\oslash (\bm{M}\bm{1}_{d}\bm{1}^{\top}_{d})$, and $\mathcal{T}_{\text{col}}(\bm{M})=\bm{M}\oslash(\bm{1}_{d}\bm{1}^{\top}_{d}\bm{M})$ as row and column-wise normalization operators of matrix $\bm{M}$, with $\oslash$ denoting the element-wise division and $\bm{1}_{d}$ a $d$ dimensions vectors of ones.
This operator has been proven that $S(\bm{M})$ must belong to the Birkhoff polytope $\mathcal{B}_{d}$ \citeapp{kuhn1955hungarian}.
In practice, we iterate up to $100$ times to obtain an approximate doubly stochastic matrix. This number of iterations is usually around $20\sim 100$ to reach a satisfactory approximation.
The whole Sinkhorn procedure is similar to the projected gradient descent algorithm.

At each permutation matrices freezing step, i.e., training the other modules of the model with permutation matrices fixed, we binarize $\{\bm{B}^{v}\}$ to obtain the closest true permutation matrix.
Specifically, we fetch a rigorous permutation from a $\bm{B}^{v}$ by solving:
\begin{equation}
    \bm{P}^{v}=\underset{\bm{P}\in\mathcal{P}_{d_{v}}}{\arg\max}\langle \bm{P},\bm{B}^{v}\rangle_{F}.
\end{equation}
Then, we apply $\{\bm{P}^{v}\}$ to each view-specific variable for obtaining the top-$d_{c}$ dimensions of each variable as common components across views.

The core of permutation learning is to implement trainable permutation matrices that ``sort" the dimensions of each view-specific variable and assure common components in the same order.
So we can select top-$d_{c}$ dimensions of each view-specific variable and obtain the common variable $\bm{c}^{*}_{t}$.

\section{Proofs}
\label{app:proofs}

\subsection{Assumptions}
\paragraph{Data-generating Process}
Let $\mathcal{Z}\in \mathbb{R}^{d}$ be the original source space and $\{\mathcal{X}^{v} \in \mathbb{R}^{D_{v}}\}$ be a set of observational spaces.
We assume that the marginal distribution $p(\bm{z}_{t})$ where $\bm{z}_{t}\in \mathcal{Z}$ is uniform:
\begin{equation}
    p(\bm{z}_{t}) = \frac{1}{|\mathcal{Z}|}.
\end{equation}
And there are series of injective functions $g^{v}:\mathcal{Z}^{v}\mapsto \mathcal{X}^{v}$ where $\cup\mathcal{Z}^{v}=\mathcal{Z}$.

Combining the causal transition function $f$ which generates the latent variable $\bm{z}_{t}$ from the previous states
\begin{equation}
    \bm{z}_{t} = f(\textbf{Pa}(\bm{z}_{t}),\bm{\epsilon}_{t}),
\end{equation}
where $\bm{\epsilon}_{t}$ is a mutual independent noise with zero mean.
From the parameters estimation perspective, $\bm{\epsilon}_{t}$ is independent of the parameters of $f$.
Hence, we denote $\tilde{\bm{z}}_{t}$ as a recursion of noiseless causal transition $\tilde{\bm{z}}_{t}=f(\textbf{Pa}(\bm{z}_{t}))$ which is also assumed to be uniformly distributed with $p(\tilde{\bm{z}}_{t})=|\mathcal{Z}|^{-1}$.
Note, the $\tilde{\bm{z}}_{t}$ and $\bm{z}_{t}$ differ by the noise term $\bm{\epsilon}_{t}$, therefore we can model the conditional distribution as:
\begin{equation}
    \begin{split}
    p(\tilde{\bm{z}}_{t}|\bm{z}_{t}) & = C^{-1}_{p}e^{-\delta(\bm{z}_{t},\tilde{\bm{z}}_{t})} \\
    \mathrm{with}\quad C_{p}(\bm{z}_{t}) & := \int e^{-\delta(\bm{z}_{t},\tilde{\bm{z}}_{t})}d\tilde{\bm{z}}_{t},
    \end{split}
\end{equation}
where $\delta(\cdot, \cdot)$ is a semi-metric.

\paragraph{Model}
Let $r:\{\mathcal{X}^{v}\}\mapsto \mathcal{Z}^{v}$, where reverse function $r$ is an abbreviation for a set of function $r^{v}$ along with the merge operator.
The reverse function encodes the observational data from different views into a unified representation $\bm{z}_{t}$.
We use $h=r\circ g$ to represent the mapping between original space $\mathcal{Z}$ and the estimated version.
The modeled conditional distribution is expressed with $h$ as $q_{h}(\tilde{\bm{z}}_{t}|\bm{z}_{t})$ and
\begin{equation}
    \begin{split}
    q_{h}(\tilde{\bm{z}}_{t}|\bm{z}_{t})= & C^{-1}_{h}(\bm{z}_{t})e^{-\delta(h(\tilde{\bm{z}}_{t}),h(\bm{z}_{t}))/\mu} \\
    \mathrm{with}\quad C_{h}(\bm{z}_{t}):= & \int e^{-\delta(h(\bm{z}_{t}),h(\tilde{\bm{z}}_{t}))/\mu}d\tilde{\bm{z}}_{t},
    \end{split}
\end{equation}
where $C_{q}(\bm{z}_{t})$ is the partition function and $\mu > 0$ is a scale constant.

For modeling $\tilde{\bm{z}}_{t}$ with the causal transition function, we introduce the noiseless parametric causal transition:
\begin{align}
    \tilde{\bm{z}}_{t} & =f(\bm{z}_{H_{t}})=\sum\nolimits^{L}_{\tau=1}\bm{A}_{\tau}\bm{z}_{t-\tau},
\end{align}

\subsection{Proof of Theorem \ref*{theorem:ce}}

First, we introduce the following lemma extended from \citeapp{zimmermannContrastiveLearningInverts2021} to show the consistency of the connection between the cross-entropy and the contrastive loss.

\begin{lemma}[$\mathcal{L}_{\rm contr}$ converges to the cross-entropy between latent distribution]\label{lemma:contrastive_ce}
    If the ground-truth marginal distribution $p$ is uniform, then for fixed $\mu > 0$, as the number of negative samples $M\to +\infty$, the (normalized) contrastive loss converges to
    \begin{equation}
        \begin{split}
        \lim_{M\to +\infty} \mathcal{L}_{\rm contr}(r;\mu,M)-\log M + \log|\mathcal{Z}|=\\
        \underset{\bm{z}_{t}\sim p(\bm{z}_{t})}{\mathbb{E}}[H(p(\cdot|\bm{z}_{t}),q_{h}(\cdot|\bm{z}_{t}))],
        \end{split}
    \end{equation}
    where $H$ is the cross-entropy between the ground-truth conditional distribution $p$ over positive pairs and a conditional distribution $q_{h}$ parameterized by the model $r$, and $C_{h}(\bm{z}_{t})\in \mathbb{R}^{+}$ is the partition function of $q_{h}$:
    \begin{equation}
        \begin{split}
            q_{h}(\tilde{\bm{z}}_{t}|\bm{z}_{t}) & = C_{h}(\bm{z}_{t})^{-1}e^{-\delta(h(\tilde{\bm{z}}_{t}),h(\bm{z}_{t}))/\mu} \\
            \mathrm{with} \quad C_{h}(\bm{z}_{t}) & := \int e^{-\delta(h(\tilde{\bm{z}}_{t}),h(\bm{z}_{t}))/\mu}.
        \end{split}
    \end{equation}
\end{lemma}

\begin{proof}
    The cross-entropy between the conditional distribution $p$ and $q_{h}$ is given by
    \begin{align}
        & \underset{\bm{z}_{t}\sim p(\bm{z}_{t})}{\mathbb{E}}\left[H(p(\cdot|\bm{z}_{t}), q_{h}(\cdot|\bm{z}_{t}))\right] \\
        =& \underset{\bm{z}_{t}\sim p(\bm{z}_{t})}{\mathbb{E}} \left[\mathbb{E}_{\tilde{\bm{z}}_{t}\sim p(\tilde{\bm{z}}_{t}|\bm{z}_{t})}\left[-\log q_{h}(\tilde{\bm{z}}_{t}|\bm{z}_{t})\right]\right] \\
        =& \underset{\tilde{\bm{z}}_{t},\bm{z}_{t}\sim p(\tilde{\bm{z}}_{t},\bm{z}_{t})}{\mathbb{E}}\left[\frac{1}{\mu}\delta(h(\tilde{\bm{z}}_{t}), h(\bm{z}_{t})) + \log C_{h}(\bm{z}_{t})\right] \\
        =& \underset{\bm{z}_{t}\sim p(\bm{z}_{t})}{\mathbb{E}} \left[\frac{1}{\mu} \underset{\tilde{\bm{z}}_{t}\sim p(\tilde{\bm{z}}_{t}|\bm{z}_{t})}{\mathbb{E}}\left[\delta(h(\tilde{\bm{z}_{t}}),h(\bm{z}_{t}))\right]+\log C_{h}(\bm{z}_{t})\right],
    \end{align}
    which can be written as
    \begin{equation}
        \frac{1}{\mu}\underset{\substack{\bm{z}_{t} \sim p(\bm{z}_{t}) \\ \tilde{\bm{z}}_{t}\sim p(\tilde{\bm{z}}_{t} | \bm{z}_{t})}}{\mathbb{E}}[\delta(h(\tilde{\bm{z}}_{t}),h(\bm{z}_{t}))] + \underset{\substack{\bm{z}_{t} \sim p(\bm{z}_{t}) \\ \tilde{\bm{z}}_{t} \sim p(\tilde{\bm{z}}_{t}|\bm{z}_{t})}}{\mathbb{E}}[\log C_{h}(\bm{z}_{t})].
    \end{equation}
    Inserting the definition of $C_{h}$ gives
    \begin{align}
        =&\frac{1}{\mu}\underset{\substack{\bm{z}_{t} \sim p(\bm{z}_{t}) \\ \tilde{\bm{z}}_{t}\sim p(\tilde{\bm{z}}_{t} | \bm{z}_{t})}}{\mathbb{E}}[\delta(h(\tilde{\bm{z}}_{t}),h(\bm{z}_{t}))] \\
        &+ \underset{\bm{z}_{t}\sim p(\bm{z}_{t})}{\mathbb{E}}\left[\log \left(\int_{\mathcal{Z}} e^{-\delta(h(\tilde{\bm{z}}_{t}),h(\bm{z}_{t}))/\mu}d\tilde{\bm{z}}_{t}\right)\right].
    \end{align}
    Next, the second term can be expaned by $|\mathcal{Z}||\mathcal{Z}|^{-1}$, yielding
    \begin{align}
        =&\frac{1}{\mu}\underset{\substack{\bm{z}_{t} \sim p(\bm{z}_{t}) \\ \tilde{\bm{z}}_{t}\sim p(\tilde{\bm{z}}_{t} | \bm{z}_{t})}}{\mathbb{E}}[\delta(h(\tilde{\bm{z}}_{t}),h(\bm{z}_{t}))]\\
        & + \underset{\bm{z}_{t}\sim p(\bm{z}_{t})}{\mathbb{E}}\left[\log \left(\int_{\mathcal{Z}} \frac{|\mathcal{Z}|}{|\mathcal{Z}|} e^{-\delta(h(\tilde{\bm{z}}_{t}),h(\bm{z}_{t}))/\mu}d\tilde{\bm{z}}_{t}\right)\right].
    \end{align}
    Finally, by using asymptotics of $\mathcal{L}_{\rm contr}$ derived by \citeapp{wangUnderstandingContrastiveRepresentation2020}
    and the marginal is uniform, i.e., $p(\tilde{\bm{z}}_{t})=|\mathcal{Z}|^{-1}$, this can be simplified as
    \begin{align}
        =&\frac{1}{\mu}\underset{\substack{\bm{z}_{t} \sim p(\bm{z}_{t}) \\ \tilde{\bm{z}}_{t}\sim p(\tilde{\bm{z}}_{t} | \bm{z}_{t})}}{\mathbb{E}}[\delta(h(\tilde{\bm{z}}_{t}),h(\bm{z}_{t}))]\\
        & + \underset{\bm{z}_{t}\sim p(\bm{z}_{t})}{\mathbb{E}}\left[\log \left(\underset{\tilde{\bm{z}}_{t}\sim p(\tilde{\bm{z}}_{t})}{\mathbb{E}} \left[ e^{-\delta(h(\tilde{\bm{z}}_{t}),h(\bm{z}_{t}))/\mu}\right] \right)\right]\\
        & + \log |\mathcal{Z}| \nonumber.\\
        =&\lim_{M\to +\infty}\mathcal{L}_{\rm contr}(r;\mu,M)-\log M + \log p|\mathcal{Z}|.
    \end{align}
\end{proof}

\begin{customtheorem}{\ref*{theorem:ce}}[Minimizers of cross-entropy recover the latent variables and their relations]
    Let latent space $\mathcal{Z}$ be a convex body in $\mathbb{R}^{d}$, $h=r\circ g:\mathcal{Z}\mapsto\mathcal{Z}$, $f$ be a causal transition function, and $\delta$ be a metric, induced by a norm. If $g$ is differentiable and injective, $r,f$ are expressive enough to be minimizers of contrastive objective $\mathcal{L}_{\rm contr}$ in Eq.~(\ref{eq:loss_contr}) for $M\to +\infty$, then, we have $h=r\circ g$ is invertible and affine mapping and $f$ captures the ground-truth causal relations.
\end{customtheorem}

\begin{proof}
    According to the Lemma~\ref{lemma:contrastive_ce}, once the positive pairs are given by $\tilde{\bm{z}}_{t}=f(\bm{z}_{H_{t}},\bm{\epsilon}_{t})$, the cross-entropy between the conditional distribution $p$ and $q_{h,f}$ is given by,
    \begin{align}
        \lim_{M\to +\infty} \mathcal{L}_{\rm contr}(r,f;\mu,M)-\log M + \log|\mathcal{Z}|= \nonumber \\
        \underset{\bm{z}_{t}\sim p(\bm{z}_{t})}{\mathbb{E}}[H(p(\cdot|\bm{z}_{t}),q_{h,f}(\cdot|\bm{z}_{t}))],
    \end{align}
    Note that $q_{h,f}(\tilde{\bm{z}}_{t}|\bm{z}_{t})$ is powerful enough to match $p(\tilde{\bm{z}}_{t}|\bm{z}_{t})$ for the correct choice of $h$ and $f$ on any point $\bm{z}_{t}$, e.g., the isometry.
    If $h$ and $f$ is minimizers of the cross-entropy $\mathbb{E}_{p(\tilde{\bm{z}}_{t}|\bm{z}_{t})}[-\log q_{h,f}(\tilde{\bm{z}}_{t}|\bm{z}_{t})]$, then we have
    \begin{equation}
        p(\tilde{\bm{z}}_{t}|\bm{z}_{t})=q_{h,f}(\tilde{\bm{z}}_{t}|\bm{z}_{t}).
    \end{equation}
    The distance metric $\delta(\cdot, \cdot)$ used in the prior conditional distribution $p$ and parameterized conditional distribution $q_{h,f}$:
    \begin{align}
        p(\tilde{\bm{z}}_{t}|\bm{z}_{t}) & =q_{h,f}(\tilde{\bm{z}}_{t}|\bm{z}_{t})\\
        \Leftrightarrow C^{-1}_{p}(\bm{z}_{t})e^{-\delta(\tilde{\bm{z}}_{t},\bm{z}_{t})} & = C^{-1}_{h}(\bm{z}_{t})e^{-\delta(h(\tilde{\bm{z}}_{t}),h(\bm{z}_{t}))/\mu} \\
        \Leftrightarrow C_{p}(\bm{z}_{t}) & = C_{h}(\bm{z}_{t}).
    \end{align}

    By expanding the formulation of $f$ with the parametric causal transition as:
    \begin{equation}
        f(\bm{z}_{H_{t}},\bm{\epsilon}_{t}) =\sum^{L}_{\tau=1}\bm{A}_{\tau}\bm{z}_{t-\tau}+\bm{\epsilon}_{t},
    \end{equation}
    and take it into the definition of conditional distribution:
    \begin{align}
        p(\tilde{\bm{z}}_{t}|\bm{z}_{t}) & =C^{-1}_{p}(\bm{z}_{t})e^{-\delta(\tilde{\bm{z}}_{t},\bm{z}_{t})} \\
        & = C^{-1}_{p}(\bm{z}_{t})e^{-\delta(f(\bm{z}_{H_{t}}),\bm{z}_{t})},
    \end{align}
    similarly, the definition of parameterized conditional distribution is:
    \begin{align}
        q_{h,f}(\tilde{\bm{z}}_{t}|\bm{z}_{t}) & =C^{-1}_{h}(\bm{z}_{t})e^{-\delta(h(\tilde{\bm{z}}_{t}),h(\bm{z}_{t}))/\mu} \\
        & = C^{-1}_{h}(\bm{z}_{t})e^{-\delta(f(h(\bm{z}_{H_{t}})),h(\bm{z}_{t}))/\mu}.
    \end{align}
    As the marginal normalization constants are identical, we obtain for all $\bm{z}_{t}\in \mathcal{Z}$
    \begin{align}
        e^{-\delta(\tilde{\bm{z}}_{t},\bm{z}_{t})} & = e^{-\delta(h(\tilde{\bm{z}}_{t}),h(\bm{z}_{t})) / \mu} \\
        \Leftrightarrow e^{-\delta(f(\bm{z}_{H_{t}}),\bm{z}_{t})} & = e^{-\delta(f(h(\bm{z}_{H_{t}})), h(\bm{z}_{t}))},
    \end{align}
    then, by the definition of $f$ which is parameterized by a series of full-rank matrices, it is bijective. 
    And $\tilde{\bm{z}}_{t}=f(\bm{z}_{H_{t}})$, hence, $q_{h,f}(\tilde{\bm{z}}_{t} | \bm{z}_{t}) = q_{h,f}(\bm{z}_{t}|\bm{z}_{H_{t}})$ we can derive the conditional distribution as:
    \begin{align}
        p(\bm{z}_{t}|\bm{z}_{H_{t}}) & = q_{h,f}(\bm{z}_{t}|\bm{z}_{H_{t}}) \\
        \Leftrightarrow \delta(f(\bm{z}_{H_{t}}),\bm{z}_{t}) & = \frac{1}{\mu}\delta(f(h(\bm{z}_{H_{t}})), h(\bm{z}_{t})) \\
        \Leftrightarrow \|\bm{\epsilon}_{t}\|_{\alpha} & = \frac{1}{\mu}\|h(\bm{\epsilon}_{t})\|_{\alpha}.
    \end{align}
    This means that $h$ preserves the $\ell_{\alpha}$-distance between points.
    Moreover, because $h$ is bijective, the Mazur-Ulam theorem tells us that $h$ must be an affine transform, i.e., by choosing proper $\mu$, there is an orthogonal matrix $\bm{U}\in \mathbb{R}^{d\times d}$ that $h(\bm{\epsilon}_{t})=\bm{U}\bm{\epsilon}_{t}$ and $h(\bm{z}_{t})=\bm{U}\bm{z}_{t}$.
    Define the delay operator $z^{-\tau}$, and extend the multichannel blind deconvolution from 
    \begin{equation}\label{eq:mbd}
        \bm{\epsilon}_{t} = \bm{z}_{t} - \tilde{\bm{z}}_{t} = \left[\bm{I}-\sum^{L}_{\tau=1}\bm{A}_{\tau}z^{-\tau}\right]\bm{z}_{t}
    \end{equation}
    take $\bm{z}_{t}=\bm{U}^{\top}h(\bm{z}_{t})$ and exchange the $\bm{\epsilon}_{t}$ and $\bm{z}_{t}$ in Eq.~(\ref{eq:mbd}), then, by extending \citeapp[Lemma 1]{zhangGeneralLinearNonGaussian2011} we have following equivalence:
    \begin{align}
        \bm{z}_{t} & = \left[\left(\bm{I}-\sum^{L}_{\tau=1}\bm{A}_{\tau}z^{-\tau}\right)^{-1}\right]\bm{\epsilon}_{t} \\ 
        & = \bm{U}^{\top}\left[\left(\bm{I}-\sum^{L}_{\tau=1}\bm{U}\bm{A}_{\tau}\bm{U}^{\top}z^{-\tau}\right)^{-1}\right]h(\bm{\epsilon}_{t}) \\
        & = \left[\left(\bm{I}-\sum^{L}_{\tau=1}\bm{A}_{\tau}z^{-\tau}\right)^{-1}\right]\bm{U}^{\top}h(\bm{\epsilon}_{t}).
    \end{align}

    Note, this result is consistent with that $h$ must be an affine transform, i.e., there is an orthogonal matrix making the estimated source be $h(\bm{z}_{t})=\bm{U}\bm{z}_{t}$ and the estimated transition matrix $\hat{\bm{A}}_{\tau} = \bm{A}_{\tau}\bm{U}^{\top}$.
    In other words, once the minimizers of contrastive loss $\mathcal{L}_{\rm contr}$ are obtained, the identifiability is guaranteed up to linear equivalence classes, and any solution within them is the best solution.
\end{proof}

Further, by chosing proper $\delta(\cdot,\cdot)$, the results extended from Banach-Lamperti Theorem \citeapp{lampertiIsometriesCertainFunctionspaces1958} indicate that we can achieve a more strong identifiability:
\begin{theorem}
    For non-isotropic norm induced metric $\ell_{\alpha}$, if $1\leq \alpha \leq \infty$ and $\alpha\neq 2$.
    A matrix $\bm{U}\in \mathbb{R}^{d\times d}$ is an isometry if and only if $\bm{U}$ is a generalized permutation matrix, there is only a channel permutation $\pi$ between estimations and ground-truth such that $\hat{z}_{i,t}=s_{i}z_{\pi(i),t}$, where $s_{i}$ is a scaling constant.
\end{theorem}

\begin{proof}
    See \citeapp{liIsometriesLpnorm1994}.
\end{proof}

\subsection{Proof of Theorem \ref{theorem:mult-view}}

\begin{customtheorem}{\ref*{theorem:mult-view}}[Minimizers of the multi-view objective maintains the channel corespondency]
    Let $\mathcal{Z}^{v}\subseteq \mathbb{R}^{d_{v}},\mathcal{Z}=\cup\mathcal{Z}^{v}$ and $\cap \mathcal{Z}^{v}\neq \emptyset$. If $\bm{B}^{v}\in \mathfrak{B}_{d_{v}}$ be a doubly stochastic matrix, $r,f,\{\bm{B}^{v}\}$ are the minimizers of contrastive objective for $M\to +\infty$, then, we have $h=r\circ g$ is an isometry, and $\{\bm{B}^{v}\}$ can rearrange the components of each view source such that common components from each view source are aligned.
\end{customtheorem}

\begin{proof}
    Since the source $\bm{z}^{v}_{t}$ of each view is first recovered by a reverse function $r^{v}$ individually, there is a distinctive transformation $h^{v}$ between each estimated view-specific source $\hat{\bm{z}}^{v}_{t}$ and true $\bm{z}^{v}_{t}$.
    We show that if and only if the Eq.~(\ref{eq:opt-permutation}) is optimized, the $r,f,\{\bm{B}^{v}\}$ are minimizers of the contrastive objective.
    Before continuing, we first show the property of a doubly stochastic matrix $\bm{B}^{v}$.
    By Birkhoff–von Neumann theorem, a doubly stochastic matrix can be represented by
    \begin{equation}
        \bm{B}^{v}=\theta_{1}\bm{P}^{v}_{1}+\cdots+\theta_{k}\bm{P}^{v}_{k},
    \end{equation}
    where $\theta_{1},\dots,\theta_{k} \geq 0$ and $\sum^{k}_{i=1}\theta_{i}=1$, and permutation matrix $\bm{P}^{v}_{i}\in \{0,1\}^{d_{v}\times d_{v}}$.
    
    Therefore, $R(\bm{B}^{v}$) is non-negative, and the minimizers of
    \begin{equation*}
        \sum^{2}_{v=1}\mathbb{E}\left[\|(\bm{B}^{v}\bm{z}^{v}_{t})_{1:d_{c}}-\bm{c}^{*}_{t}\|^{2}\right]+\eta R(\bm{B}^{v}),
    \end{equation*}
    can only be obtained by enforcing both the first and second terms to be zero.
    The first term is achieved when $(\bm{B}^{v}\hat{\bm{z}}^{v}_{t})_{1:d_{c}} = \bm{c}^{*}_{t}$ for all views, and the second term is achieved when $\bm{B}^{v}$ converges to a permutation matrix such that $R(\bm{B}^{v})=-\sum_{i,j} B_{i,j}\log B_{i,j}=0$.
    
    Then, the common source estimated from different views differ in a permutation transformation, i.e., a channel rearrangement,
    \begin{equation}
        (\bm{P}^{1}\hat{\bm{z}}^{1}_{t})_{1:d_{c}} = (\bm{P}^{2}\hat{\bm{z}}^{2}_{t})_{1:d_{c}} = \bm{c}^{*}_{t}.
    \end{equation}
    If $\bm{c}^{*}_{t}$ was not the common source, there would be some components of private source entangled in $\bm{c}^{*}_{t}$ that are equivalent on every point.
    This would mean that these components are actually same on all points, which contradicts the assumption that $\bm{c}^{*}_{t}$ is not the common source.
    
    Further, this result also indicates that the common source and the private source are disentangled, thereby the merge operator that splits and re-concatenates the estimated sources to construct the complete source is guaranteed to be consistent with the optimal condition.
\end{proof}

By the Theorem~\ref{theorem:ce} $h$ is an affine transformation and the common source and the private sources are disentangled.
For a source that generates two views, there is an orthogonal matrix $\bm{U}$ such that:
\begin{align}
    \hat{\bm{z}}_{t} & = \bm{U}\bm{z}_{t} \\
    & = \begin{bmatrix}
        \bm{U}_{c} & \bm{0} & \bm{0} \\
        \bm{0} & \bm{U}_{1} & \bm{0} \\
        \bm{0} & \bm{0} & \bm{U}_{2}\\
        \end{bmatrix}
        \begin{bmatrix}
            \bm{c}_{t} \\
            (\bm{B}^{1}\bm{z}^{1}_{t})_{d_{c}+1:d_{1}} \\
            (\bm{B}^{2}\bm{z}^{2}_{t})_{d_{c}+1:d_{2}} \\
        \end{bmatrix},
\end{align}
where $\bm{U}_{c}\in \mathbb{R}^{d_{c}\times d_{c}}$, $\bm{U}_{1}\in \mathbb{R}^{d_{1}-d_{c}\times d_{1}-d_{c}}$ and $\bm{U}_{2}\in \mathbb{R}^{d_{2}-d_{c}\times d_{2}-d_{c}}$ are orthogonal matrices.
Moreover, when $\ell_{\alpha}$ is non-isotropic norm induced metric, the $\bm{U}_{c}$, $\bm{U}_{1}$ and $\bm{U}_{2}$ can also be narrowed down to the permutation matrices.

\section{Practical Implementations}
\label{app:experiments}

In this section, we describe the details of the practical implementation of more experiments.
All experiments are conducted on two workstations with 8 NVIDIA RTX 2080Ti GPUs.

\subsection{Handle More Views}
As mentioned, the MuLTI can be easily extended to more than 2 views.
Herein, we conduct experiments on \textbf{Multi-view VAR} dataset under different numbers of views $V\in\{1, 2, 3, 4\}$, we set the dimensions of ground-truth latent variables to $d=10$, time lags to $L=2$, noise distribution follow the Laplacian with $\lambda=0.05$.
Setting details for each scenario are as follows:
\begin{description}
    \item [$V=1$]: $d_{c}=d=10$
    \item [$V=2$]: $d_{c}=4,d_{1}=7,d_{2}=7$
    \item [$V=3$]: $d_{c}=4,d_{1}=6,d_{2}=6,d_{3}=6$
    \item [$V=4$]: $d_{c}=2,d_{1}=4,d_{2}=4,d_{3}=4,d_{4}=4$
\end{description}
\begin{figure}[ht]
    \centering
    \includegraphics[width=0.8\linewidth]{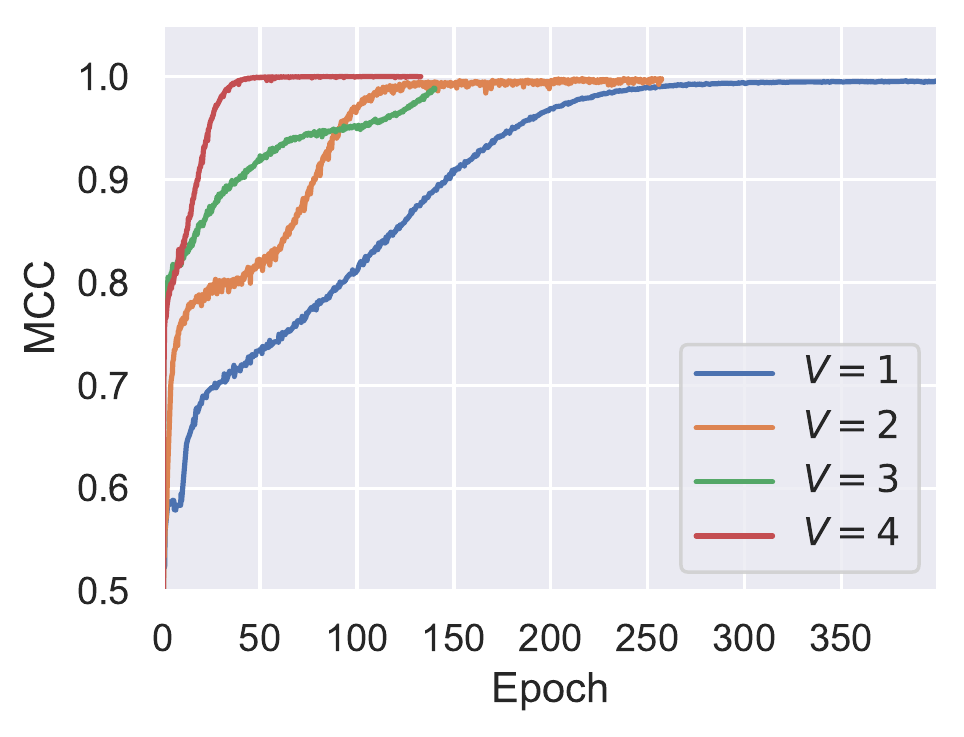}
    \caption{MCC trajectories with different numbers of views.}
    \label{fig:supp-more-views-mcc}
\end{figure}
In the case of $V=1$, no further merge operation is required; therefore, the permutation learning module is removed.
According to the results, MuLTI achieves full identifiability in all scenarios, thus MuLTI can handle datasets with more than two views successfully.

\subsection{Handle Higher Dimension}
We would like to add the results larger than the general setting, by setting $40\% $ dimensions to be shared, d=10,16,20,50.
The results are shown in Table~\ref{tab:dimension-ablation}.
\begin{table}[h]
\centering
\caption{Albation results with higher dimension}
\label{tab:dimension-ablation}
\begin{tabular}{@{}l|llll@{}}
\toprule
Source Dimension & $d=10$  & $d=16$  & $d=20$  & $d=50$  \\ \midrule
MCC              & 99.80 & 99.90 & 99.92 & 96.58 \\ \bottomrule
\end{tabular}%
\end{table}

\subsection{More Ablation Results}

\subsubsection{Effect of hyper-parameter}
The hyper-parameters for loss include $\beta_{1}$, $\beta_{2}$, and $\beta_{3}$, which are weights of each term in the augmented total objective.
While the effects of $\beta_{1},\beta_{2}$ are rather obvious than $\beta_{3}$, we perform a grid search evaluation strategies on $\beta_{1}\in\{0.0, 0.01, 0.02, 0.05, 0.1\},\beta_{2}\in\{0.0, 0.01, 0.02, 0.05, 0.1\}$.
The results are reported in Figure~\ref{fig:supp-beta-grid}.
The final MCC scores are typically around $99.0\%$, indicating that model performance is fairly robust to choices of $\beta_{1},\beta_{2}$.
But, when $\beta_{1}$ or $\beta_{2}$ is set to $0$, we observe the performance drops notably, which indicates the importance of our residual minimization and merge operations.
\begin{figure}
    \centering
    \includegraphics[width=0.9\linewidth]{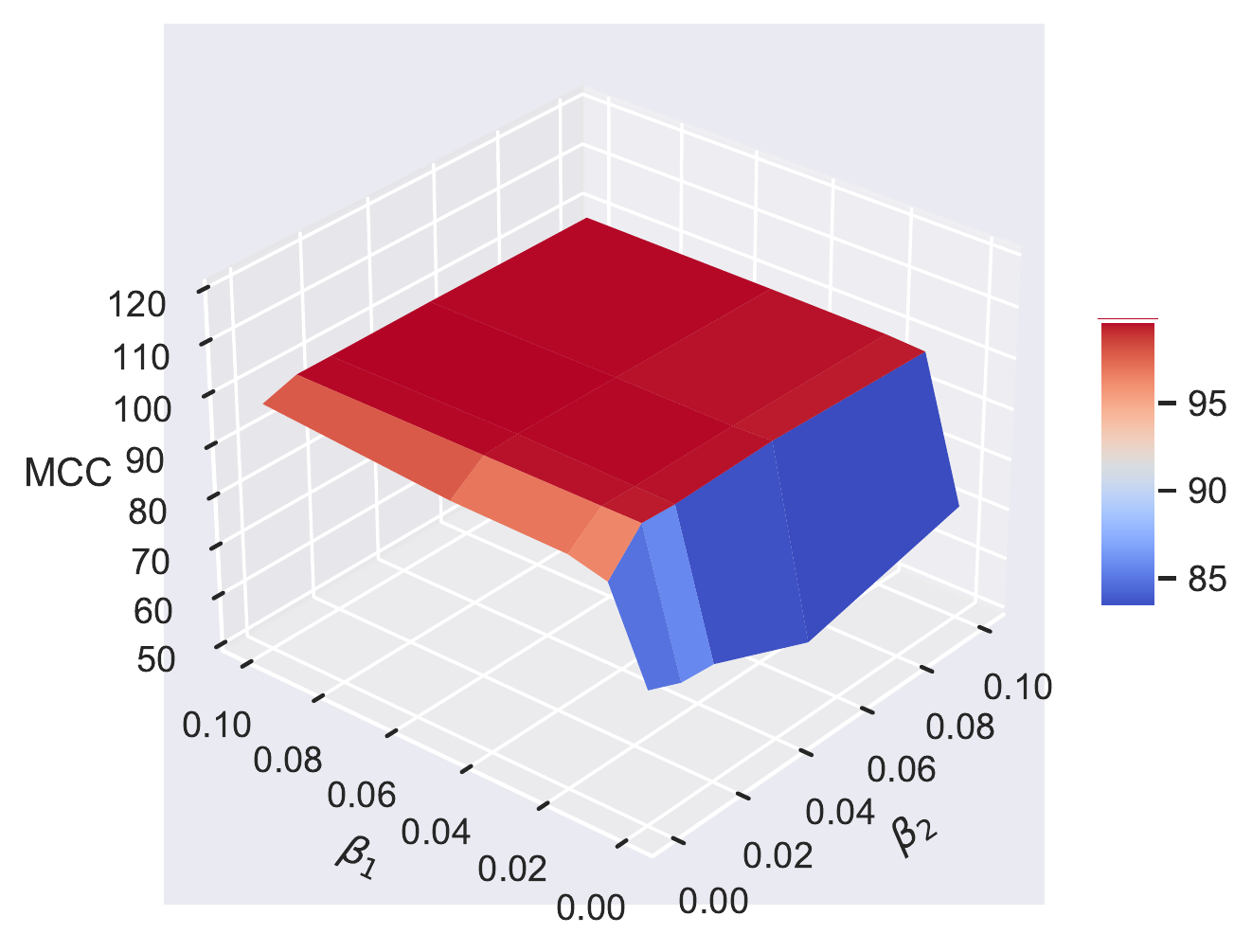}
    \caption{Results of grid hyper-parameter $\beta_{1},\beta_{2}$ evaluation.}
    \label{fig:supp-beta-grid}
\end{figure}

\subsubsection{Effect of different causal transition time lags}
We verify the robustness of our MuLTI under different time lags $L\in\{2, 4, 6\}$ with $\beta_{1}=0.01, \beta_{2}=0.01, \beta_{3}=1e-5$. The results are shown in Figure~\ref{fig:supp-more-lags-mcc}.
From the results, we can see that although the model takes more steps when the total time lags are $L=6$, the final result shows that our method can achieve full identifiability in these different experimental scenarios.

\begin{figure}
    \centering
    \includegraphics[width=0.8\linewidth]{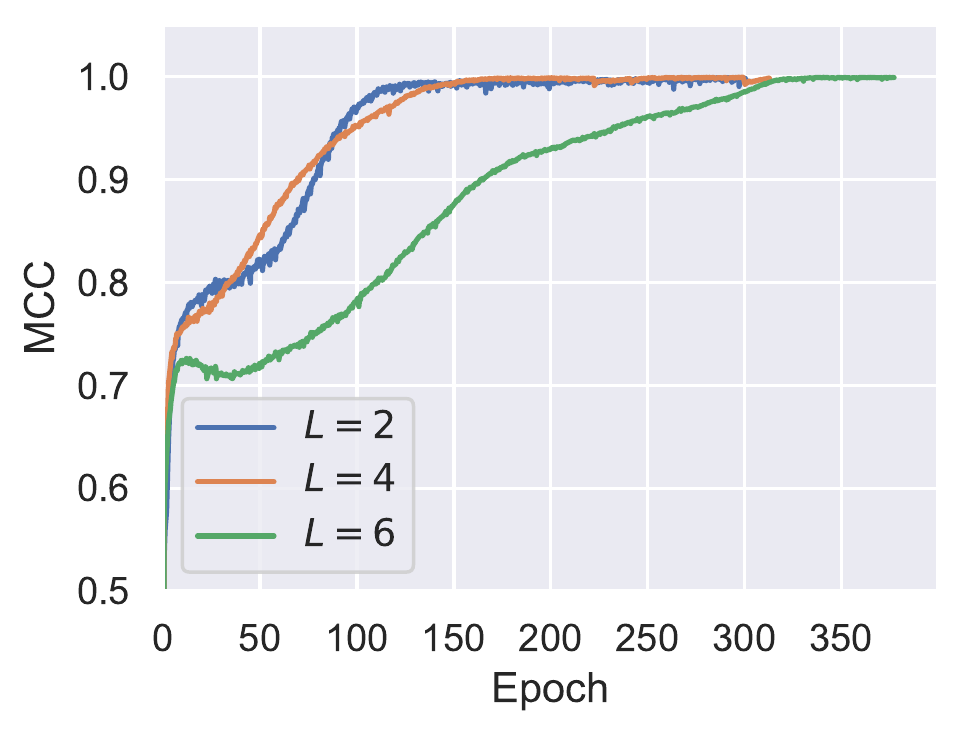}
    \caption{MCC trajectories with different time lags.}
    \label{fig:supp-more-lags-mcc}
\end{figure}

\subsection{Time Complexity}

It is true that the original OT algorithm suffers from a high time complexity problem, however, the Sinkhorn algorithm can be efficiently implemented on GPUs. We clock the experiments with different D sizes [10, 20, 50], and the total running time of Sinkhorn/complete training of 800 epochs are shown in Table~\ref{tab:time-complexity}. It can be shown that the Sinkhorn iterations take less than 1/10 time cost that regular model training.
\begin{table}[h]
\centering
\caption{Total running time (in hours) of the Sinkhorn-Knopp iterations and model training on GPUs.}
\label{tab:time-complexity}
\begin{tabular}{@{}l|lll@{}}
\toprule
Source Dimension & d=10 & d=20 & d=50 \\ \midrule
Sinkhorn-Knopp   & 0.30 & 0.69 & 0.89 \\
Model Training   & 3.31 & 6.71 & 8.81 \\ \bottomrule
\end{tabular}
\end{table}

\subsection{Details of Datasets Configuration}

\subsubsection{Synthetic VAR} is a synthetic dataset modified from \citeapp{yaoLearningTemporallyCausal2021}, which involves VAR processes and corresponding nonlinear multi-view observations.
To generate latent process, we first select transition matrices $\{A_{\tau}\}^{L}_{\tau=1}$ from uniform distribution $U[-0.5, 0.5]$, then sample initial states $\{z_{t}\}^{L}_{t=1}$ from normal distribution for generating sequences $\{z_{t}\}^{T}_{t=1}$.
To generate multi-view time series, we randomly select $d_{1}$ and $d_{2}$ dimensions of latent variables for each view.
The nonlinear observations of each view are created by mixing latent variables with random parameterized MLPs $g^{v}(\cdot)$ following previous work \citeapp{hyvarinenNonlinearICATemporally2017}.
Specifically, each $g^{v}(\cdot)$ is composed of three hidden fully connected layers and Leaky ReLU units.

\subsubsection{Mass-spring system} is a classical physical system that follows Hooke's law. In this dataset, eight balls are simulated inside a 2D box where they may collide elastically with its walls. A spring is connected to each pair of variables with a uniform probability.

\begin{equation*}
    f_{i j}=-\Delta_{k}\left(\bm{y}_{i}-\bm{y}_{j}\right), \quad \ddot{\bm{y}}_{i}=\sum_{j=1}^{N} f_{i j}, \quad \bm{p}_{i}=\left\{\bm{y}_{i}, \dot{\bm{y}}_{i}\right\}
\end{equation*}
The $f_{ij}$ represents the unidirectional interaction between a ball $j$ and a ball $i$, with $\Delta_{k}$ denoting the type of edge for each pair of variables, and $\bm{y}_{i}$ and $\dot{\bm{y}}_{i}$ denoting the ball's 2D position and velocity, respectively. The continuous variable $\bm{p}_{i}$ is constructed by concatenating the position and the velocity. A step-by-step integration of the previous differential equations is used to simulate the trajectory once the initial location and velocity have been sampled. Trajectories are obtained by subsampling every 50 steps at a step size of 0.001.
In practice, we create the scenario of simulation and generate the trajectories with 
\emph{Pymunk}, then render the trajectories into two views with \emph{Matplotlib}.
We set the rest length of the spring to $[5, 10]$, and we set the stiffness of the spring in relation to $20$.
Additionally, external forces are added to the ball at each time step, which are sampled from a Laplacian noise distribution with $\lambda=0.1$.

\subsection{Pseudo-Code of MuLTI}
We summarize the pseudo-code of our MuLTI method in Algorithm~\ref{alg:algorithm}.
\begin{algorithm}[ht]
\caption{Pesudo-code for MuLTI}
\label{alg:algorithm}
\textbf{Input}: Training dataset $\{\mathcal{X}^{v}\}$, encoders $\{r^{v}\}$, transition function $f$, matrices $\{\bm{B}^{v}\}$.\\
\textbf{Parameter}: $\beta_{1},\beta_{2},\beta_{3},\mu,\eta$\\
\textbf{Output}: Trained encoders $\{r^{v}\}$, transition functions $f$, permutation matrices $\bm{B}^{v}$.
\begin{algorithmic}[1] 
\STATE Let $t=0$;
\WHILE{Not converge}
\STATE Sample a mini-batch $\{\bm{x}^{v}\}$;
\STATE Obtain $\{\bm{z}^{v}_{t}=r^{v}(\bm{x}^{v}_{t})\}$;
\STATE Merge $\bm{z}_{t}=\bm{m}(\{\bm{z}^{v}_{t}\})$ from each $\{\bm{z}^{v}_{t}\}$;
\STATE Construct $\{(\mathcal{H}_{t},\bm{z}_{t})\}$ from $\{\bm{z}_{t}\}$;
\STATE Fetch causal transited variable $\tilde{\bm{z}}_{t}=f(\mathcal{H}_{t})$;
\STATE Construct contrastive pairs $(\bm{z}_{t},\tilde{\bm{z}}_{t}),(\bm{z}_{t},\bm{z}^{-}_{t})$;
\STATE Obtain $\{\hat{\epsilon}_{i,t}\}$ and $\{\hat{\epsilon}_{i,t}\}^{\text{perm}}=\text{randperm}(\{\hat{\epsilon}_{i,t}\})$;
\IF {Do contrastive learning}
\STATE Compute $\mathcal{L}_{\text{contr}}$ according to Eq.\eqref{eq:loss_contr};
\STATE $\mathcal{L}_{\epsilon}=\mathbb{E}_{\{(\bm{z}_{t},\tilde{\bm{z}}_{t})\}} \delta(\bm{z}_{t},\tilde{\bm{z}}_{t})$;
\STATE $\mathcal{L}_{m}=\sum_{v}\mathbb{E}[\|(\bm{B}^{v}\bm{z}^{v}_{t})_{1:D_{c}}-\bm{c}^{*}_{t}\|^{2}]$;
\STATE $\mathcal{L}_{\mathcal{D}}=\log \mathcal{D}(\{\hat{\epsilon}_{i,t}\}) - \log (1-\mathcal{D}(\{\hat{\epsilon}_{i,t}\}^{\text{perm}}))$;
\STATE $\mathcal{L}=\mathcal{L}_{\text{contr}}+\beta_{1}\mathcal{L}_{m}+\beta_{2}\mathcal{L}_{\epsilon}+\beta_{3}\mathcal{L}_{\mathcal{D}}$;
\STATE Minimize loss $\mathcal{L}$;
\ELSIF {Do permutation learning}
\STATE $\mathcal{L}_{m}=\sum_{v}\mathbb{E}[\|(\bm{B}^{v}\bm{z}^{v}_{t})_{1:D_{c}}-\bm{c}^{*}_{t}\|^{2}]+\eta R(\bm{B}^{v})$;
\STATE Minimize loss $\mathcal{L}_{m}$;
\ELSIF {Do discriminator learning}
\STATE $\mathcal{L}_{\mathcal{D}}'=\log \mathcal{D}(\{\hat{\epsilon}_{i,t}\}) + \log (1-\mathcal{D}(\{\hat{\epsilon}_{i,t}\}^{\text{perm}}))$;
\STATE Minimize loss $\mathcal{L}_{\mathcal{D}}'$;
\ENDIF
\ENDWHILE
\STATE \textbf{return} $\{r^{v}\},\{\bm{B}^{v}\},f,\{\bm{z}_{t}\}$.
\end{algorithmic}
\end{algorithm}

\newpage
\bibliographystyleapp{named}
\bibliographyapp{ijcai23}

\end{document}